\documentclass[11pt]{article}

\usepackage{amsmath,amsfonts,amsthm,amssymb,color}
\usepackage{fullpage}

\usepackage{microtype}
\usepackage{graphicx}
\usepackage{subfigure}
\usepackage{booktabs} 
\usepackage{hyperref}

\usepackage{algorithm}
\usepackage{algorithmic}
\newcommand{\citet}{\cite}
\newcommand{\citep}{\cite}

\usepackage{amsmath,amsfonts,amsthm,amssymb,color}
\usepackage{xspace}

\newtheorem{theorem}{Theorem}[section]
\newtheorem{corollary}[theorem]{Corollary}
\newtheorem{lemma}[theorem]{Lemma}

\newtheorem{definition}[theorem]{Definition}

\newtheorem{remark}[theorem]{Remark}

\newcommand{\eps}{\ensuremath{\epsilon}\xspace}
\renewcommand{\tilde}{\widetilde}
\renewcommand{\hat}{\widehat}

\newcommand{\eqdef}{:=}
\newcommand{\R}{\mathbb{R}}

\newcommand{\PP}{\mathcal{P}}
\newcommand{\KK}{\mathcal{K}}
\newcommand{\CC}{\mathcal{C}}
\newcommand{\II}{\mathcal{I}}

\DeclareMathOperator{\tr}{tr}
\newcommand{\norm}[1]{\left\|#1\right\|}
\newcommand{\normone}[1]{\norm{#1}_1}
\newcommand{\normtwo}[1]{\norm{#1}_2}
\newcommand{\norminf}[1]{\norm{#1}_\infty}

\providecommand{\expect}[2]{\ensuremath{\ifthenelse{\equal{#1}{}}{\mathbb{E}}{\mathbb{E}_{#1}}\!\left[#2\right]}\xspace}
\providecommand{\prob}[2]{\ensuremath{\ifthenelse{\equal{#1}{}}{\Pr}{\Pr_{#1}}\!\left[#2\right]}\xspace}

\newenvironment{lp*}{\begin{equation*}  \begin{array}{lll}}{\end{array}\end{equation*}}

\newcommand{\inner}[1]{\langle #1\rangle}

\usepackage{enumitem}

\newcommand{\NN}{\mathcal{N}}
\newcommand{\mus}{\mu^\star}
\newcommand{\abs}[1]{\left|{#1}\right|}
\DeclareMathOperator{\diag}{diag}

\DeclareMathOperator{\smax}{smax}
\DeclareMathOperator{\prox}{prox}

\title{High-Dimensional Robust Mean Estimation via Gradient Descent}

\author{
\begin{tabular}{c c}
  \begin{tabular}{c}
Yu Cheng\thanks{
Part of the work was done while visiting the Institute for Advanced Study.}\\
University of Illinois at Chicago\\
{\tt yucheng2@uic.edu}
  \end{tabular} &
  \begin{tabular}{c}
Ilias Diakonikolas\thanks{Supported by NSF Award CCF-1652862 (CAREER), a Sloan Research Fellowship, and 
a DARPA Learning with Less Labels (LwLL) grant. Part of this work was done while visiting the Simons Institute 
for the Theory of Computing during the Summer 2019 program on the Foundations of Deep Learning.}\\
University of Wisconsin-Madison\\
{\tt ilias@cs.wisc.edu}
  \end{tabular} \\ \\
  \begin{tabular}{c}
Rong Ge\thanks{Supported by NSF CCF1704656, NSF CCF-1845171 (CAREER), NSF CCF-1934964, a Sloan Fellowship, and a Google Faculty Research Award.
Part of the work was done while visiting the Institute for Advanced Study.}\\
Duke University\\
{\tt rongge@cs.duke.edu}
  \end{tabular} &
  \begin{tabular}{c}
Mahdi Soltanolkotabi\thanks{Supported by the Packard Fellowship in Science and
Engineering, a Sloan Research Fellowship in Mathematics, NSF CCF-CIF grants \#1846369 and \#1813877, AFOSR-YIP under award \#FA9550-18-1-0078, DARPA Learning with Less Labels (LwLL) and Fast Network Interface Cards (FastNICs) programs, and a Google faculty research award. Part of the work was done while visiting the Simons Institute for the Theory of Computing.}\\
University of Southern California\\
{\tt soltanol@usc.edu}
  \end{tabular}
\end{tabular}
}

\begin{document}

%
%
%
%
%
%
%
%


\maketitle 

\begin{abstract}
We study the problem of high-dimensional robust mean estimation in the presence of a constant fraction of adversarial outliers. A recent line of work has provided sophisticated polynomial-time algorithms for this problem with dimension-independent error guarantees for a range of natural distribution families.

In this work, we show that a natural non-convex formulation of the problem can be solved directly by gradient descent. Our approach leverages a novel structural lemma, roughly showing that any approximate stationary point of our non-convex objective gives a near-optimal solution to the underlying robust estimation task. Our work establishes an intriguing connection between algorithmic high-dimensional robust statistics and non-convex optimization, which may have broader applications to other robust estimation tasks.
\end{abstract}

\thispagestyle{empty}
\setcounter{page}{0}

\makeatletter{}

\newpage

\section{Introduction} \label{sec:intro}

Learning in the presence of outliers is an important goal in
machine learning that has become
a pressing challenge in a number of high-dimensional data analysis applications, 
including data poisoning attacks~\cite{Barreno2010,BiggioNL12, SteinhardtKL17} 
and exploratory analysis of real datasets with natural outliers, 
e.g., in biology~\cite{RP-Gen02, Pas-MG10, Li-Science08}.
In both these application domains, the outliers are not ``random'' but can
be arbitrarily correlated, and could exhibit rather complex structures that 
is essentially impossible to accurately model.
Hence, the goal in these settings is to design computationally 
efficient estimators that can tolerate a small constant fraction 
of arbitrary outliers.

Throughout this paper, we focus on the following data contamination model
that generalizes several existing models, including 
Huber's contamination model~\cite{Huber64}. 

\begin{definition}[Strong Contamination Model] \label{def:adv}
Given a parameter $0< \eps < 1/2$ and a distribution family $\mathcal{D}$ on $\R^d$,
the \emph{adversary} operates as follows: The algorithm specifies the 
number of samples $N$, and $N$ samples are drawn from some unknown $D \in \mathcal{D}$.
The adversary is allowed to inspect the samples, remove up to $\eps N$ of them
and replace them with arbitrary points. This modified set of $N$ points is then given as input
to the algorithm. We say that a set of samples is {\em $\eps$-corrupted}
if it is generated by the above process.
\end{definition}
The parameter $\eps$ in the above definition is the fraction of corrupted samples
and quantifies the power of the adversary. Intuitively, among our samples,
an unknown $(1-\eps)$ fraction are generated from a
distribution of interest and are called {\em inliers}, and the rest are called {\em outliers}.

The statistical foundations of outlier-robust estimation were laid out in early
work by the robust statistics community, starting with the pioneering works 
of \citet{Tukey60} and \citet{Huber64}. In contrast, until fairly recently, even the most basic algorithmic questions were poorly understood.
Specifically, even for the basic task of high-dimensional mean estimation, 
all known robust estimators had runtime exponential in the dimension, rendering
them ineffective in high-dimensional settings. 

Recently, \citet{DKKLMS16, LaiRV16}
gave the first efficiently computable robust estimators 
for high-dimensional unsupervised learning
tasks, including mean and covariance estimation. Specifically, \citet{DKKLMS16}
obtained the first polynomial-time robust estimators with {\em dimension-independent} error guarantees,
i.e., with error scaling only with the fraction of corrupted samples $\eps$ and not with the dimensionality of the data.
Since the dissemination of these works, there has been a flurry of research activity
on algorithmic aspects of high-dimensional robust statistics; see, e.g.,~\citet{DK19-survey} for a recent survey on the topic.

Despite this exciting progress, the design of efficient robust estimators
in high dimensions remains challenging. The difficulty, of course, lies in the non-convexity 
of the underlying optimization problem. 
Prior work developed fairly sophisticated algorithmic tools, 
even for the task of robust mean estimation.  These include convex relaxations~\cite{DKKLMS16} 
and quite subtle iterative spectral methods~\cite{DKKLMS16, LaiRV16}. 

A natural and important goal is to understand 
to what extent such sophisticated methods are indeed necessary
or whether much simpler robust learning algorithms exist. 
In this work, we take a direct optimization view of these problems 
and ask the following general question: 
\begin{quote}
\it
Is it possible to solve robust estimation tasks
by standard first-order methods?
\end{quote}
We believe that this question merits investigation in its own right.
Moreover, its positive resolution may
have significant implications in the practical adoption of robust estimation methods. 
Particularly so since prior algorithms are either 
(1) computationally prohibitive (relying on large convex relaxations), 
(2) involve carefully crafted parameters that require precise tuning for practical deployment, 
or (3) are challenging to extend to more sophisticated robust estimation tasks.
A tantalizing possibility is the following: For a range of high-dimensional robust 
estimation tasks, there exists a (natural) non-convex formulation
such that gradient descent efficiently converges 
to a near-optimal solution. 

{\em In this paper, we show that this premise is true 
for the task of high-dimensional robust mean estimation.}
In robust mean estimation, we are given a set of $N$ $\eps$-corrupted
samples from an unknown distribution $D$ in a known family $\mathcal{D}$,
and we want to output a hypothesis vector $\widehat{\mu}$ such that 
$\|\widehat{\mu} - \mus\|_2$ is as small as possible, where $ \mus$ is the mean of $D$. 
For simplicity, we will assume in this discussion 
that $D$ is an unknown mean and identity covariance Gaussian
on $\R^d$. We note that our results hold under more general distributional assumptions,
as in~\citet{DKKLMS16, DKK+17}.

The goal in robust mean estimation is to develop efficient algorithms whose $\ell_2$-error guarantee
scales only with $\eps$ and not with the dimension $d$.
In particular, for the identity covariance Gaussian case, \citet{DKKLMS16} gave polynomial-time algorithms for the problem
that use $N = \tilde{\Omega}(d/\eps^2)$ samples and guarantee error $O(\eps \sqrt{\log(1/\eps)})$.
This error guarantee matches known Statistical Query (SQ) lower bounds~\cite{DKS17-sq}.

\subsection{Overview of Results and Contributions} \label{ssec:results}
In this paper, we consider a natural non-convex optimization formulation of high-dimensional
robust mean estimation, and show that gradient descent\footnote{Throughout, we informally 
use the term ``gradient descent'' to refer to variations of gradient descent methods, 
which involve updates based on a generalized notion of a gradient, e.g.,~sub-gradient for non-differentiable functions.} 
efficiently converges to a near-optimal solution.
Specifically, we show that gradient descent converges in a polynomial number of iterations and matches the error guarantee 
of the best known polynomial-time algorithms for the problem. Our technical contribution lies 
in showing that {\em any} approximate stationary point of our non-convex objective suffices -- 
in the sense that it gives a near-optimal solution for the underlying estimation problem.

To describe our non-convex formulation, we require some background. We use 
the following framework for robust mean estimation, introduced in~\citet{DKKLMS16}.
The idea is to assign a non-negative weight to each data point and then find
an appropriate combination of weights such that the weighted empirical mean
is close to the true mean. The constraint on the chosen weights is that
they represent at least a $(1-\eps)$-density fractional subset of the dataset. 
More formally, given datapoints $X_1, \ldots, X_N \in \R^d$ with corresponding 
data matrix $X \in \R^{d \times N}$, the objective is to find a weight vector
$w \in \R^N$ such that $\mu_w = X w$ is close to $\mus$. The constraint on $w$
is that it belongs in the set 
\[
\Delta_{N,\eps} = \left\{ w \in \R^N : \normone{w}=1 \text{ and } 0 \le w_i \le \frac{1}{(1-\eps) N}, \, \forall i \right\} \; ,
\]
which is the convex hull of all uniform distributions over subsets $S \subseteq [N]$ of size $|S| = (1-\eps)N$.

\citet{DKKLMS16} established a key structural lemma (Lemma~\ref{thm:dkklms}), which formed the basis of their algorithms.
Roughly speaking, the lemma states that any weight vector $w$ is a good solution if the spectral norm
of the weighted empirical covariance, $\Sigma_w = \sum_{i=1}^N w_i (X_i - \mu_w)(X_i - \mu_w)^\top$, is small.
This lemma directly motivates the following non-convex optimization formulation:
\begin{equation}
\label{eqn:spectral-formulation}
\textrm{Min } \; \|\Sigma_w\|_2 \; \textrm{ subject to } \; w \in \Delta_{N,2\eps}
\end{equation}
It follows from the aforementioned structural lemma that a near-optimal solution $w$
to \eqref{eqn:spectral-formulation} gives an $\mu_w$ that is close to $\mu^{\ast}$. 
The challenge is that the objective function is not convex, hence it is unclear 
how to efficiently optimize. Faced with this difficulty, prior works
on the topic~\cite{DKKLMS16, DKK+17} developed various sophisticated algorithms. 

In this paper, we work directly with the natural formulation \eqref{eqn:spectral-formulation}.
Despite its non-convexity, we are able to leverage the structure of the problem to 
show that gradient descent efficiently converges
to a good vector $w$. In more detail, we prove a novel result about the structure 
of approximate stationary points of this objective. 

\begin{theorem}[informal statement] \label{thm:inf-spectral-struct}
Any approximate stationary point $w$ of \eqref{eqn:spectral-formulation} defines 
an $\mu_{w}$ that is close to $\mus$.
\end{theorem}

See Theorem~\ref{thm:no-bad-local-opt} for a detailed formal statement.
Technically speaking, our statement is more subtle for various reasons, including the fact 
that the objective function is not differentiable and the domain is constrained. As a result, we require
a careful definition of stationarity in our setting. 

Given Theorem~\ref{thm:inf-spectral-struct}, we proceed to show that projected sub-gradient 
descent converges to an approximate stationary point in a polynomial number
of iterations. This step is also somewhat intricate 
as the function is non-convex, non-smooth and the optimization problem \eqref{eqn:spectral-formulation} involves constraints.
In summary, we establish the following theorem:

\begin{theorem}\label{thm:final-spectral}
After $\tilde{O}(N^2d^4)$ iterations, projected sub-gradient descent on \eqref{eqn:spectral-formulation} outputs a
point $w$ such that with high probability $\|\mu_{w} - \mus \|_2 = O(\eps\sqrt{\log(1/\eps)})$.
\end{theorem}

The bound we establish on the convergence rate
on the spectral norm objective \eqref{eqn:spectral-formulation} is polynomially bounded, but relatively slow.
Our second main contribution involves considering the ``softmax'' version of the spectral norm, which has 
better smoothness properties. An analogous lemma about the structure of stationary points allows
us to show a faster rate of convergence for this modified objective.

\begin{theorem}\label{thm:final-softmax}
After $\tilde O(N d^3 / \eps)$ iterations, projected gradient descent on the softmax objective outputs a
point $w$ such that with high probability $\|\mu_{w} - \mus \|_2 = O(\eps\sqrt{\log(1/\eps)})$.
\end{theorem}
As evident from the above result, the additional smoothness of the ``softmax" objective 
allows us to establish a significantly improved bound on the number of iterations.

\subsection{Related Work} \label{ssec:prior}
The algorithmic question of designing efficient
robust mean estimators in high-dimensions has been extensively studied
in recent years. After the initial papers~\cite{DKKLMS16, LaiRV16}, 
a number of works~\cite{DKK+17, SteinhardtCV18, 
ChengDR18, DHL19, DepLec19, CDGW19} have obtained algorithms 
with improved asymptotic worst-case runtimes that work 
under weaker distributional assumptions on the good data. 
Moreover, efficient high-dimensional robust mean estimators
have been used as primitives for robustly solving a range of machine learning tasks
that can be expressed as stochastic optimization 
problems~\cite{PrasadSBR2018, DiakonikolasKKLSS2018sever}.

We compare our approach with the works of~\citet{ChengDR18} and 
\citet{DHL19} that give the asymptotically fastest known algorithms for robust mean estimation.
At a high-level,~\citet{ChengDR18}, building on the convex programming relaxation 
of~\citet{DKKLMS16}, proposed a primal-dual approach for robust mean estimation that  
reduces the problem to a poly-logarithmic number of packing and covering SDPs.
Each such SDP is known to be solvable in time $\tilde{O}(N d)$, using mirror descent
~\citet{AllenLO16, PengTZ16}. \citet{DHL19} build on the iterative
spectral approach of~\citet{DKKLMS16}. That work uses the matrix multiplicative weights update
method with a specific regularization and dimension-reduction to improve the worst-case runtime. 

In contrast to all of the above, we use a natural non-convex formulation of the robust mean estimation task, 
and show that a standard first-order method provably
and efficiently converges to a near-optimal solution. Even though the convergence rates
that we establish in this work do not yield the fastest known asymptotic runtimes for the problem,
we believe that our approach is conceptually interesting for a number of reasons.
First, our theorem regarding stationary points provides novel structural understanding
about robust mean estimation and can be viewed as an explanation as to why this problem
is polynomially solvable. Second, it is plausible that gradient descent applied in this context 
is more stable than previously known algorithms and may facilitate the adoption of robust estimation
methods in practice. 
We hope that this work will serve as the starting point for solving
other robust estimation tasks via first-order methods. 

Finally, we note that there is an increasing literature on developing rigorous guarantees for non-convex optimization problems via gradient descent, e.g.,~see the recent survey \cite{jain2017non} for a review of this literature. With a few exceptions \cite{loh2011high, hassani2017gradient}, this literature mostly focuses on showing that gradient descent converges to a global optimum starting from a spectral \cite{keshavan2010matrix, candes2015phase, tu2015low} or random initialization \cite{ge2015escaping} in settings where there are no bad local optima. In contrast to most of this literature, 
in this paper we show that any \emph{stationary} point has good approximation properties so that no specialized or random initialization is necessary. We believe that such a perspective may enable rigorous analysis of many other non-convex optimization problems.

\subsection{Roadmap}
\label{ssec:roadmap}
In Section \ref{sec:prelims}, we set up the necessary notation and provide some background on robust mean estimation. 
In the next two sections, we focus on the spectral norm objective. In Section \ref{sec:struc}, we prove our main structural result showing 
that any stationary point of the spectral norm objective yields a good solution. We also extend this result in Appendix \ref{apx:algo-proof}, showing that in fact, any approximate stationary point yields a sufficiently good solution. In Section \ref{sec:algo}, we show that gradient descent converges to an approximate stationary point and hence yields a good solution in a polynomial number of iterations. In Appendix~\ref{sec:softmax}, we prove structural and algorithmic results for the softmax objective, showing that any approximate stationary point of the softmax objective yields a good solution, and we can find an approximate stationary point using projected gradient descent in a polynomial number of iterations. 
We conclude with future directions in Section~\ref{sec:conc}.

\makeatletter{}
\section{Preliminaries and Background} \label{sec:prelims}
\noindent {\bf Notation.}
For $N \in \mathbb{Z}_+$, we denote $[N] \eqdef \{1, \ldots, N\}.$
For a vector $x$, we use $\normone{x}$, $\normtwo{x}$, and $\norminf{x}$ to denote the $\ell_1$, $\ell_2$, and $\ell_\infty$ norm of $x$ respectively.
For a matrix $A$, we use $\normtwo{A}$ to denote the spectral norm of $A$.

For two vectors $x, y \in \R^{n}$, we use $x^\top y = \sum_{i=1}^n x_i y_i$ to denote the inner product of $x$ and $y$, and we use $x \odot y \in \R^{n}$ to denote entrywise product of $x$ and $y$.
For a vector $x \in \R^{n}$, let $\diag(x) \in \R^{n \times n}$ denote a diagonal matrix with $x$ on the diagonal.
For a matrix $A \in \R^{n \times n}$, let $\diag(A) \in \R^{n}$ denote a column vector with the diagonal entries of $A$.

Let $I$ denote the identity matrix.
For a matrix $A \in \R^{n \times n}$, let $\tr(A)$ denote the trace of $A$.
For two matrices $A$ and $B$ of the same dimensions, let $A \bullet B = \inner{A, B} = \tr(A^\top B)$ be the entry-wise inner product of $A$ and $B$.
We use $\exp(A)$ to denote the matrix exponential of $A$.

A symmetric matrix $A \in \R^{n \times n}$ is said to be positive semidefinite (PSD) if $x^\top A x \ge 0$ for all $x \in \R^n$.
For two symmetric matrices $A$ and $B$, we write $A \preceq B$ iff the matrix $B - A$ is positive semidefinite.
Let $\Delta_{n\times n}$ be the set of all PSD matrices of trace $1$.

\paragraph{Framework.} 
We use $N$ for the number of input samples, $d$ for the dimension of the ground-truth distribution, 
and $\eps$ for the fraction of corrupted samples.
Given $N$ datapoints $X_1, \ldots, X_N \in \R^d$, we use $X \in \R^{d \times N}$ 
to denote the sample matrix, where the $i$-th column of $X$ is  $X_i$.

Given $w \in \R^N$, let $\mu_w = X w = \sum_{i=1}^N w_i X_i$ denote the weighted empirical mean and let $\Sigma_w = \sum_{i=1}^N w_i(X_i - \mu_w)(X_i - \mu_w)^\top$ 
denote the weighted empirical covariance.
Let $\Delta_{N,\eps}$ denote the convex hull of all uniform distributions over subsets $S \subseteq [N]$ 
of size $|S| = (1-\eps)N$:
\[
\Delta_{N,\eps} = \left\{ w \in \R^N : \normone{w}=1 \text{ and } 0 \le w_i \le \frac{1}{(1-\eps) N}, \, \forall i \right\} \; .
\]
Every weight vector $w \in \Delta_{N,\eps}$ corresponds to a fractional set of $(1-\eps)N$ samples.

\paragraph{Background on Robust Mean Estimation.}
As mentioned in the introduction, our non-convex formulation
is directly motivated by the following structural lemma:

\begin{lemma}[\citet{DKKLMS16}]
\label{thm:dkklms}
Let $S$ be an $\eps$-corrupted set of $N = \tilde \Omega(d / \eps^2)$ samples from an unknown $\NN(\mus, I)$
and $w \in \Delta_{N,2\eps}$. If $\lambda_{\max}\left(\Sigma_w\right) \leq 1 + \delta$, for some $\delta \geq 0$,
then with high probability, we have that $\normtwo{\mus - \mu_w} = O(\sqrt{\eps \delta} + \eps \sqrt{\log(1/\eps)})$.
\end{lemma}

As in prior work, we will establish correctness for our algorithms
under deterministic conditions on the inliers (good samples) 
that hold with high probability.
Let $G^\star$ denote the original set of $N$ good samples.
Let $S = G \cup B$ denote the input samples after the adversary replaced $\eps$-fraction of the samples, 
where $G \subset G^\star$ is the set of remaining good samples and $B$ 
is the set of bad samples (outliers) added by the adversary.
Note that $|G| = (1-\eps) N$ and $|B| = \eps N$.
Given $w \in \R^N$, let $w_G = \sum_{i \in G} w_i$ be the total weight on good samples, 
and $w_B$ be the total weight on bad samples.

We require the following concentration bounds to hold for the original $N$ good samples $G^\star$ 
(which happens with high probability when $N = \tilde \Omega(d/\eps^2)$).
For all $\hat w \in \Delta_{N, 3\eps}$, we require the following condition to hold for $\delta = O(\eps \log(1/\eps))$:
\begin{align}
\normtwo{\sum_{i \in G^\star} \hat w_i (X_i - \mus) (X_i - \mus)^\top - I} \le \delta \;.
\label{eqn:good-moments-original}
\end{align}

Condition~\eqref{eqn:good-moments-original} on \emph{original} samples $G^\star$ implies the following conditions on the \emph{remaining} good samples $G$.
For any weight vector $w \in \Delta_{N,2 \eps}$ on the $\eps$-corrupted set of samples $S = G \cup B$:
\begin{align}
\normtwo{\sum_{i \in G} w_i (X_i - \mus) (X_i - \mus)^\top - I} \le \delta \;.
\label{eqn:good-moments-remaining}
\end{align}
This is because we can define $\hat w$ as follows: $\hat w_i = \frac{w_i}{w_G}$ for all $i \in G$ and $\hat w_i = 0$ for all $i \in B$.
Since $w \in \Delta_{N,2 \eps}$, we have $\norminf{\hat w} \le \frac{\norminf{w}}{w_G} = \frac{\norminf{w}}{1-w_B} \le \frac{\norminf{w}}{1-|B|\cdot\norminf{w}} \le \frac{1}{(1-3\eps)N}$.
In other words, $\hat w \in \Delta_{N, 3\eps}$ and Condition~\eqref{eqn:good-moments-remaining} follows directly from 
Condition~\eqref{eqn:good-moments-original}.

\begin{remark}[Distributional Assumptions] \label{rem:distr}
{\em For simplicity, in this paper we focus on the fundamental setting that the good data are drawn from
an unknown mean and identity covariance Gaussian distribution. It should be noted that our structural and algorithmic 
results hold under more general distributional assumptions. 
Specifically, Theorem~\ref{thm:pgd} immediately applies
to identity covariance subgaussian distributions, with the same error guarantees, 
since it only relies on the concentration bounds~\eqref{eqn:good-moments-original}
and~\eqref{eqn:good-moments-remaining} that only require subgaussian tails (see, e.g.,~\cite{DKK+17}.)
Moreover, one can modify the proof of our structural results (Theorems~\ref{thm:no-bad-local-opt} and~\ref{lem:moreau-good}), 
mutatis-mutandis, to apply (1) for distributions with bounded covariance (i.e., $\Sigma \preceq I$) and match the optimal $O(\sqrt{\eps})$ 
approximation to the mean~\cite{DKK+17}; and, (2) more generally, under the $(\eps, \delta)$-stability condition 
of~\cite{DK19-survey} to yield an $O(\delta)$ $\ell_2$-approximation to the mean.}
\end{remark}

\paragraph{Background and Definitions of Stationarity.}
Note that the spectral norm is not a differentiable function and 
therefore we need an alternative definition of stationarity. 
To address this issue, by the definition of spectral norm, we can define a 
function $F(w, u) = u^\top \Sigma_w u$ that takes two parameters as input: 
the weights $w \in \R^N$ and a unit vector $u \in \R^d$. 
Our non-convex objective $\min_w f(w):=\normtwo{\Sigma_w}$ is then equivalent 
to solving the minimax problem $\min_w \max_{u} F(w, u)$. The function $\max_{u} F(w,u)$ is weakly-convex, 
and we use the following stationary point definition that is common in the weakly-convex optimization 
literature \cite{rockafellar1970convex, rockafellar1981favorable, drusvyatskiy2017proximal, davis2018stochastic, jin2019local}.

\begin{definition}[First-order stationary point]
\label{def:localopt-spectral}
Let $F(w, u)$ be a function that is differentiable with respect to $w$ for all $u$.
Let $f(w) = \max_u F(w, u)$.
Consider the constrained optimization problem $\min_{w \in K} f(w)$, where $K$ is a closed convex set.
We say that $w \in K$ is a first-order stationary point if there exists some $u \in \arg\max_v F(w, v)$ such that
\[
(\nabla_w F(w, u))^\top (\tilde w - w) \ge 0 \; \text{ for all } \; \tilde w \in K \; .
\]
\end{definition}

We also need a notion of an {\em approximate} stationary point in the sense that the updates from one iteration to the next do not change much. 
In the unconstrained and differentiable case, such a point can be characterized by the gradient being small. 
However, the objective function we consider is both non-differentiable and has constraints, 
so that a proper definition of approximate stationarity is much more subtle. To overcome this, 
we appeal to tools from conic geometry and notions of stationarity for weakly convex functions \cite{rockafellar1970convex, rockafellar1981favorable, drusvyatskiy2017proximal, davis2018stochastic} to define an appropriate notion of approximate stationarity.

To discuss the notion of approximate stationarity that we use, 
we need to work with a smoothed variant of the objective known as the {\em Moreau envelope}.
\begin{definition}[Moreau envelope]
\label{def:moreau}
For any function $f$ and closed convex set $\KK$, its associated Moreau envelope
$f_\beta(w)$ is defined to be the function
\[
f_\beta(w) := \min_{\tilde w \in \KK} f(\tilde w) + \beta \normtwo{w - \tilde w}^2 \; .
\]
\end{definition}
The Moreau envelope can be thought of as a form of convolution between the original function $f$ 
and a quadratic, so as to smoothen the landscape. In particular, when $f(w)$ takes the form 
of a maximization problem ($f(w)=\max_{u} F(w,u)$) with $F$ a mapping that is $\beta$-smooth 
in the $u$ parameter ($|\nabla_w F(w,\widetilde{u})-\nabla_w F(w,u)|\le \beta \|\widetilde{u}-u\|_2$), 
the Moreau envelope is also $\beta$-smooth \cite{drusvyatskiy2017proximal}. Therefore, the approximate 
stationarity of the Moreau envelope can be easily defined through its gradient allowing us to define the following notion of approximate stationarity.
\begin{definition}[Approximate first-order stationary point]\label{approxstationarity} For any function $f$ and closed convex set $\KK$ consider its associated Moreau envelope $f_\beta(w)$ per Definition \ref{def:moreau}. we say that 
a point $w$ is a $\rho$-approximately stationary point if $\|\nabla f_\beta(w)\|_2 \le \rho$.
\end{definition}
As mentioned earlier, 
the spectral norm admits a minimax formulation 
of the form $f(w)=\max_{u} F(w,u)$. Furthermore, as detailed in Appendix~\ref{apx:algo-proof}, 
the corresponding function $F(w,u)$ is $\beta$-smooth with $\beta=2\|X\|_2^2$, 
so that this notion of approximate stationarity can be applied to the objective of interest in this paper.

\makeatletter{}
\section{Structural Result: Any Approximate Stationary Point Suffices}
\label{sec:struc}

In this section, we establish our main structural result, which says that every {\em approximate} 
stationary point of \eqref{eqn:spectral-formulation} must give a $\mu_w$ that is close to $\mus$. 
For simplicity of the exposition, in the main body of this paper, we state and prove 
a simpler theorem showing that every (exact) stationary point is a good solution.

\begin{theorem}[Any stationary point is a good solution]
\label{thm:no-bad-local-opt}
Let $S$ denote an $\eps$-corrupted set of $N$ samples drawn from a $d$-dimensional Gaussian $\NN(\mus, I)$ with unknown mean $\mus$.
Suppose that $S$ satisfies Lemma~\ref{thm:dkklms} and Condition~\eqref{eqn:good-moments-remaining}.

Let $f(w)$ be the objective function defined in Equation~\eqref{eqn:spectral-formulation}.
For any first-order stationary point $w \in \Delta_{N,2\eps}$ of $f(w)$, we have $\normtwo{\mu_w - \mus} = O(\eps \sqrt{\log(1/\eps)})$.
\end{theorem}

We note that while Theorem~\ref{thm:no-bad-local-opt} shows that any (exact) 
stationary point has small objective value, a stronger statement is required for our algorithmic results in the next section. 
Specifically, we require that any {\em approximate} stationary point --- in the sense of Definition \ref{approxstationarity} --- 
which gradient descent efficiently converges to, also has low objective value. This is accomplished in the next theorem 
which we prove in Appendix \ref{apx:algo-proof}. Specifically, by appealing to the gradient of the Moreau envelope 
from Definition \ref{def:moreau}, we extend the proof of Theorem~\ref{thm:no-bad-local-opt} to show the following:

\begin{theorem}[Any approximate stationary point suffices]
\label{lem:moreau-good}
Consider the same setting as in Theorem~\ref{thm:no-bad-local-opt}.
Consider the spectral norm objective $f(w)=\|\Sigma_w\|_2$ with $f_\beta(w)$ denoting the corresponding Moreau envelope function per Definition~\ref{def:moreau} with $\beta=2\|X\|_2^2$. 
Then, for any $w \in \Delta_{N,2\eps}$ satisfying
\[
\normtwo{\nabla f_\beta(w)} = O(\log(1/\eps)) \;,
\]
we have $\normtwo{\mu_w - \mus} = O(\eps \sqrt{\log(1/\eps)})$.
\end{theorem}

In the remainder of this section, we focus on proving Theorem \ref{thm:no-bad-local-opt} 
and briefly discuss how this proof can be generalized to prove Theorem \ref{lem:moreau-good}. 
Our proof is carried out in two steps: (1) We establish a structural lemma which states 
that every stationary point $w$ must satisfy a {\em bimodal subgradient} property; 
(2) We show any point satisfying such property must have a small objective value. 
Given these two steps, we can conclude any stationary points $\mu_w$ is close to $\mus$, by Lemma~\ref{thm:dkklms}.

For the first step, the bimodal subgradient property states that 
there exists a vector $\nu \in\partial f(w)$ (in the sub-gradient of the function at that stationary point) 
 whose entries divided in two groups of indices such that for any $i\in S^{-}$ 
and any $j\in S^{+}$ we have $\nu_i\le \nu_j$. Intuitively, $S^-$ contains all 
indices with positive $w_i$, so they can potentially be decreased; while $S^+$ contains all indices 
with $w_i < \frac{1}{(1-2\epsilon) N}$, so they can potentially be increased. If the bimodal sub-gradient property is violated, 
there must be indices $i\in S^{-}$, $j\in S^{+}$, 
where $\nu_i > \nu_j$. In this case, decreasing $w_i$ and increasing $w_j$ 
would decrease the objective and thus violate stationarity.

For the second step, recall that
\[ \Sigma_w = \left(X \diag(w) X^\top - X w w^\top X^\top\right) \]
 and $F(w, u) = u^\top \Sigma_w u$.
Let us first compute the sub-gradient $\nabla_w F(w,u)$ with respect to a vector $u$:
\begin{align}
\label{eqn:nabla-Fwu}
\nabla_w F(w, u) &= X^\top u \odot X^\top u - 2 (u^\top X w) X^\top u \; .
\end{align}
Our key observation is that the sub-gradient at direction $u$ is equivalent 
to the gradient of $w$ for the one-dimensional problem with input $(X_i^\top u)_{i=1}^N$.
This allows us to effectively reduce our problem to a one-dimensional robust mean estimation problem. 
This reduction allows us to show that when the objective function is large, then 
there must be some non-zero weights associated with the corrupted points 
that are far away from the mean (these points will be in $S^-$); 
while on the other hand, $S^+$ must contain at least $\epsilon$-fraction of the good points. 
One can then select indices from these two sets to violate the bimodal sub-gradient property.

Fix a first-order stationary point $w \in \Delta_{N,2\eps}$.
Definition~\ref{def:localopt-spectral} implies that there is a corresponding unit vector $u \in \R^d$ 
such that $w$ is a stationary point of $F(w, u)$. We first state the bimodal sub-gradient property. 
\begin{lemma}[Bimodal sub-gradient property at stationarity]
\label{lem:kkt-w}
Fix $w \in \Delta_{N,2\eps}$ and a unit vector $u$ with $u^\top \Sigma_w u = \normtwo{\Sigma_w}$.
Let $S_- = \{i: w_i > 0\}$ and $S_+ = \{i : w_i < \frac{1}{(1-2\eps)N} \}$ denote the coordinates of $w$ 
that can decrease and increase respectively.
If $w$ is a first-order stationary point of $F(w, u)$, then
\[
\nabla_w F(w, u)_i \le \nabla_w F(w, u)_j \;,
\]
for all $i \in S_-$ and $j \in S_+$.
\end{lemma}
\begin{proof}
Suppose there is some $i \in S_-$ and $j \in S_+$ such that $\nabla_w F(w, u)_i > \nabla_w F(w, u)_j$, then intuitively we can make $f(w)$ smaller by decreasing $w_i$ and increasing $w_j$.
Formally, let $\tilde w = w + \min(w_i, \frac{1}{(1-2\eps)N} - w_j)(e_j - e_i)$ where $e_i$ is the $i$-th basis vector.
We have $\tilde w \in \Delta_{N, 2\eps}$ and $(\nabla_w F(w, u))^\top (\tilde w - w) < 0$, 
which violates the assumption that $w$ is a stationary point (Definition~\ref{def:localopt-spectral}).
\end{proof}

Given Lemma~\ref{lem:kkt-w}, we prove Theorem~\ref{thm:no-bad-local-opt} by contradiction.
We show that if $\mu_w$ is far from $\mus$, then $w$ violates the property stated in Lemma~\ref{lem:kkt-w} 
and therefore cannot be a stationary point.
More specifically, we show that, if $\mu_w$ is far from $\mus$, then there exists 
a bad sample with index $j \in S_-$ whose gradient is large (Lemma~\ref{lem:bad-large-gr}).
Meanwhile, the concentration bounds in Condition~\eqref{eqn:good-moments-remaining} 
guarantee that there exists a good sample with index $i \in S_+$ whose gradient is small (Lemma~\ref{lem:good-small-gr}).

\begin{lemma}[Bad sample with large gradient]
\label{lem:bad-large-gr}
Assume that Condition~\eqref{eqn:good-moments-remaining} and Lemma~\ref{thm:dkklms} hold.
Fix $w \in \Delta_{N,2\eps}$ and a unit vector $u$ with $u^\top \Sigma_w u = \normtwo{\Sigma_w}$.
Let $r = \normtwo{\mu_w - \mus}$ and suppose $r \ge c_2 \eps \sqrt{\ln(1/\eps)}$.
Then there exists some $i \in (B \cap S_-)$ such that
\[
\nabla_w F(w, u)_i - u^\top {\mus}(\mus - 2\mu_w)^\top u > 2 c_3 \cdot \frac{r^2}{\eps^2} \; .
\]
Here, $c_2 $ and $c_3$ are universal positive constants.
\end{lemma}

\begin{lemma}[Good sample with small gradient]
\label{lem:good-small-gr}
Consider the same setting as in Lemma~\ref{lem:bad-large-gr}.
There is some $j \in (G \cap S_+)$ such that
\[
\nabla_w F(w, u)_j - u^\top {\mus}(\mus - 2\mu_w)^\top u \le c_3 \cdot \frac{r^2}{\eps^2} \; .
\]
\end{lemma}

We defer the proofs of Lemmas~\ref{lem:bad-large-gr}~and~\ref{lem:good-small-gr} 
to Sections~\ref{sec:proof-bad-gr}~and~\ref{sec:proof-good-gr}, and we first use these two lemmas 
to prove Theorem~\ref{thm:no-bad-local-opt}.

\begin{proof}[Proof of Theorem~\ref{thm:no-bad-local-opt}]
Suppose that $w \in \Delta_{N,2\eps}$ is a first-order stationary point of $f(w)$, 
and moreover, $w$ is a bad solution where $\normtwo{\mu_w - \mus} \ge c_2 \eps \sqrt{\ln(1/\eps)}$.
By Definition~\ref{def:localopt-spectral}, there exists a unit vector $u \in \R^d$ such that $w$ is a stationary point of $F(w, u)$.

Fix such a vector $u$.
Since Condition~\eqref{eqn:good-moments-remaining} and Lemma~\ref{thm:dkklms} both hold, we can invoke Lemmas~\ref{lem:bad-large-gr}~and~\ref{lem:good-small-gr} on $(w, u)$ to find two coordinates $i \in S_-$ and $j \in S_+$ that violate the bimodal subgradient condition in Lemma~\ref{lem:kkt-w}.
Consequently, $w$ cannot be a stationary point of $F(w, u)$.
This leads to a contradiction, and therefore, all first-order stationary points of $f(w)$ are good solutions.
\end{proof}

We now briefly comment on the modifications required to prove Theorem \ref{lem:moreau-good} (see Appendix~\ref{apx:algo-proof}).
Theorem \ref{lem:moreau-good} is proven by first showing (using conic geometry) 
that for such an approximate stationary point an approximate bimodal sub-gradient property holds. 
Specifically, we show that the bimodal sub-gradient property (Lemma~\ref{lem:kkt-w}) is {\em stable} 
in the sense that for an approximate stationary point an \emph{approximate bimodal sub-gradient} property holds, 
i.e.,~$\nu_i\le \nu_j+\delta$. Further, for any point obeying such an approximate bimodal property, the objective is small and has good approximation guarantees. 
The last two steps when combined show that any approximate stationary point has good approximation guarantees (similar to the proof of Theorem \ref{thm:no-bad-local-opt} for exact stationary points).

\subsection{Finding a Bad Sample With Large Gradient}
\label{sec:proof-bad-gr}
In this subsection, we prove Lemma~\ref{lem:bad-large-gr}.

Lemma~\ref{lem:bad-large-gr} states that when $\mu_w$ is far from $\mus$, there exists an index $i \in (B \cap S_-)$ such that the gradient $\nabla_w F(w, u)_i$ is relatively large.

Recall that $\nabla_w F(w, u)$ in Equation~\eqref{eqn:nabla-Fwu} is the same as the gradient of the variance (weighted by $w$) of the one-dimensional samples $\left(X_i^\top u\right)_{i=1}^N$.
Roughly speaking, for this one-dimensional problem, a sample far from the (projected) true mean should have large gradient.
Our objective is to find such a sample with positive weight.

More specifically, since $w$ is a bad solution and $u$ is in the top eigenspace of $\Sigma_w$, the weighted empirical variance of the projected samples is very large.
Because the good samples cannot have this much variance, most of the variance comes from the bad samples.
We show that among the bad samples that contribute a lot to the variance, one of them must be very far from the (projected) true mean.

In this section and Section~\ref{sec:proof-good-gr}, we use $c_1, \ldots, c_4$ to denote universal constants that are independent of $N$, $d$, and $\eps$.
We give a detailed description of how to set these constants in Appendix~\ref{apx:const}.

\begin{proof}[Proof of Lemma~\ref{lem:bad-large-gr}]
We first show that the variance of one-dimensional samples $\left(X_i^\top u\right)_{i=1}^N$ is relatively large.

By Lemma~\ref{thm:dkklms}, we know that if $\normtwo{\mu_w - \mus} \ge r$ and $r \ge c_2 \eps \sqrt{\ln(1/\eps)}$, then
\[
\lambda_{\max}(\Sigma_w) \ge 1 + c_4 \cdot \frac{r^2}{\eps}
\]
for some universal constant $c_4$.

Because $u$ is a unit vector that maximizes $u^\top \Sigma_w u$, we have
\[
u^\top \Sigma_w u = \lambda_{\max}(\Sigma_w) \ge 1 + \frac{c_4 r^2}{\eps} \; .
\]
Recall that $\Sigma_w = \sum_{i=1}^N w_i (X_i - \mu_w)(X_i - \mu_w^\top)$.
If we replace $\mu_w$ with $\mus$, we have
\[
\sum_{i=1}^N w_i (X_i - \mus)(X_i - \mus)^\top \, \succeq \, \Sigma_w \; ,
\]
and therefore,
\[
u^\top \left(\sum_{i=1}^N w_i (X_i - \mus)(X_i - \mus)^\top \right)u \ge 1 + \frac{c_4 r^2}{\eps} \; .
\]

Next we show that most of this variance is due to bad samples.
By Condition~\eqref{eqn:good-moments-remaining},
\[
u^\top \left(\sum_{i \in G} w_i (X_i - \mus)(X_i - \mus)^\top\right) u \le 1 + c_1 \eps \ln(1/\eps) \; .
\]
Consequently,
\begin{align*}
 \quad u^\top \left(\sum_{i \in B} w_i (X_i - \mus)(X_i - \mus)^\top\right) u
&\ge \frac{c_4 r^2}{\eps} - c_1 \eps \ln(1/\eps) \ge 0.98 \cdot c_4 \cdot \frac{r^2}{\eps} \; .
\end{align*}
The last step is because $r \ge c_2 \cdot \eps\sqrt{\ln(1/\eps)}$ and we can choose $c_4$ to be sufficiently large.

Now that we know most of the variance is due to the bad samples, 
observe that the total weight $w_B$ on the bad samples is at most $\eps N \cdot \frac{1}{(1-2\eps)N} \le 2 \eps$.
Therefore, there must be some $i \in B$ with $w_i > 0$ such that
\begin{align*}
u^\top \left((X_i - \mus)(X_i - \mus)^\top\right) u
  &\ge \frac{0.98 \cdot c_4 \cdot r^2 \cdot \eps^{-1}}{w_B}
  \ge 0.49 \cdot c_4 \cdot \frac{r^2}{\eps^2} \; .
\end{align*}
In other words,
\[
\abs{u^\top (X_i - \mus)} \ge 0.7 \cdot \sqrt{c_4} \cdot \frac{r}{\eps} \; .
\]
By definition, $i \in B \cap S_{-}$.
It remains to show that $\nabla_w F(w, u)_i$ is large.
\begin{align*}
&\quad \nabla_w F(w, u)_i - u^\top {\mus}(\mus - 2\mu_w)^\top u \\
  &= u^\top \left((X_i - \mus)(X_i - \mus)^\top\right) u 
   - 2 u^\top \left((X_i - \mus)(\mu_w - \mus)^\top\right) u \\
  &\ge \left(u^\top (X_i - \mus)\right)^2 - 2 \abs{u^\top (X_i - \mus)} \cdot \normtwo{\mu_w - \mus} \\ 
  &\ge \frac{0.49 \cdot c_4 \cdot r^2}{\eps^2} - 2 \cdot \frac{0.7 \cdot \sqrt{c_4} \cdot r}{\eps} \cdot r
  > 2 c_3 \cdot \frac{r^2}{\eps^2} \; .
\end{align*}
The first inequality is by Cauchy-Schwarz.
The last step uses the fact that $\eps$ is sufficiently small.
\end{proof}

\subsection{Finding a Good Sample With Small Gradient}
\label{sec:proof-good-gr}
In this subsection, we prove Lemma~\ref{lem:good-small-gr}.

Lemma~\ref{lem:good-small-gr} states that there exists an index $j \in (G \cap S_+)$ such that the gradient $\nabla_w F(w, u)_j$ is relatively small.
Similar to the previous section, a sample close to the (projected) true mean should have small gradient.
Our goal is to find such a sample for which we can increase its weight.

Recall that $S^+$ contains all samples whose weight can be increased.
We first prove that there are at least $\eps N$ good samples in $S^+$.
Among these $\eps N$ good samples, the concentration bounds imply that some $X_j$ must be very close to the (projected) true mean. 
\begin{proof}[Proof of Lemma~\ref{lem:good-small-gr}]
Recall that $S^+$ contains every coordinate $i$ where $w_i < \frac{1}{(1-2\eps)N}$.
Since at most $(1-2\eps)N$ samples can have the maximum weight $\frac{1}{(1-2\eps)N}$, we know that $|S^+| \ge 2\eps N$.
Combining this with $|G| = (1-\eps)N$, we know that $|G \cap S^+| \ge \eps N$.

Fix a subset $G^+ \subseteq (G \cap S^+)$ of size $|G^+| = \eps N$.
We first show that on average, samples in $G^+$ do not contribute much to the variance.

Let $w'$ be the uniform weight vector on $G$, i.e., $w'_i = \frac{1}{(1-\eps)N}$ for all $i \in G$ and $w'_i = 0$ otherwise.
Since $w' \in \Delta_{N,2\eps}$, by Condition~\eqref{eqn:good-moments-remaining},
\[
\normtwo{\sum_{i \in G} \frac{1}{|G|} (X_i - \mus)(X_i - \mus)^\top - I} \le c_1 \cdot \eps \ln(1/\eps) \; .
\]

Let $w''$ be the uniform weight vector on $S \setminus G^+ = (G \setminus G^+) \cup B$, i.e., $w''_i = \frac{1}{(1-\eps)N}$ for all $i \in ((G \setminus G^+) \cup B)$ and $w''_i = 0$ otherwise.
Since $w'' \in \Delta_{N,2\eps}$, again by Condition~\eqref{eqn:good-moments-remaining}, we have
\[
\normtwo{\sum_{i \in G \setminus G^+} \frac{1}{|G|} (X_i - \mus)(X_i - \mus)^\top - I} \le c_1 \eps \ln(1/\eps) .
\]

Combining the previous two concentration bounds,
\begin{align*}
&\quad \normtwo{\sum_{i \in G^+} \frac{1}{|G|} (X_i - \mus)(X_i - \mus)^\top} \\
  &\le \normtwo{\sum_{i \in G} \frac{1}{|G|} (X_i - \mus)(X_i - \mus)^\top - I}
    + \normtwo{\sum_{i \in G \setminus G^+} \frac{1}{|G|} (X_i - \mus)(X_i - \mus)^\top - I} \\
  &\le 2 c_1 \cdot \eps \ln(1/\eps) \; .
\end{align*}
Consequently, because $u$ is a unit vector,
\[
u^\top \left(\sum_{i \in G^+} \frac{1}{|G|} (X_i - \mus)(X_i - \mus)^\top \right) u \le 2 c_1 \eps \ln(1/\eps) \; .
\]
At this point, we know samples in $G^+$ do not contribute much to the variance.
We now proceed to show that one of these samples satisfies the lemma.

Let $j = \arg\min_{i \in G^+} \abs{u^\top (X_i - \mus)}$.
We have
\begin{align*}
u^\top \left((X_j - \mus)(X_j - \mus)^\top \right) u &\le \frac{|G|}{|G^+|} \cdot 2 c_1 \cdot \eps \ln(1/\eps)
  \le 2 c_1 \ln(1/\eps) \; .
\end{align*}

Finally, because $\abs{u^\top (X_j - \mus)} \le \sqrt{2 c_1 \ln(1/\eps)}$, we can show that $\nabla_w F(w, u)_j$ is small:
\begin{align*}
&\quad \nabla f(w)_j - {\mus}^\top Y (\mus - 2\mu_w) \\
  &= u^\top \left((X_j - \mus)(X_j - \mus)^\top\right) u
     + 2 u^\top \left((X_j - \mus)(\mu_w - \mus)^\top\right) u \\
  &\le 2 c_1 \ln(1/\eps) + 2 \sqrt{2 c_1 \ln(1/\eps)} \cdot r \\
  &\le \frac{c_3}{2} \cdot \frac{r^2}{\eps^2} + \frac{c_3}{2} \cdot \frac{r}{\eps} \cdot r \le c_3 \cdot \frac{r^2}{\eps^2} \; .
\end{align*}
The last step uses that $c_3$ is sufficiently large, as well as the fact that $\ln(1/\eps) \le \frac{r^2}{\eps^2}$ because $r \ge c_2 \eps \sqrt{\ln(1/\eps)}$.
\end{proof} 

\makeatletter{}
\section{Algorithmic Result: Finding a Stationary Point via Gradient Descent}
\label{sec:algo}

In this section, we show that a simple Projected Gradient Descent (PGD) algorithm (Algorithm~\ref{alg:spectral}) 
can efficiently find an approximate stationary point $w$ of our spectral norm objective, 
and that $w$ is a good solution to our robust mean estimation task.

\begin{algorithm}[h]
\caption{Robust Mean Estimation via PGD}
\label{alg:spectral}
\begin{algorithmic}
  \STATE {\bf Input:} $\eps$-corrupted set of $N$ samples $\{X_i\}_{i=1}^N$ on $\R^d$ satisfying Condition~\eqref{eqn:good-moments-remaining}, and $\eps < \eps_0$.
  \STATE {\bf Output:} $w \in \R^N$ with $\normtwo{\mu_w - \mus} \le O(\eps \sqrt{\log(1/\eps)})$.
  \STATE Let $F(w, u) = u^\top \Sigma_w u$.
  \STATE Let $w_0$ be an arbitrary weight vector in $\Delta_{N,2\eps}$.
  \STATE Let $T = \tilde O(N^2 d^4)$.
  \FOR{$\tau=0$ {\bf to} $T-1$}
    \STATE Find a unit vector $u_\tau \in \R^d$ such that $F(w_\tau, u_\tau) \ge (1-\eps) \max_u F(w_\tau, u)$.
    \STATE $w_{\tau+1} = \PP_{\Delta_{N,2\eps}}\left(w_\tau - \eta \nabla_w F(w_\tau, u_\tau) \right)$, where $\PP_\KK(\cdot)$ is the $\ell_2$ projection operator onto $\KK$.
  \ENDFOR \\
  \STATE {\bf return } $w_{\tau^\star}$ where $\tau^\star = \arg\min_{\, 0 \le \tau < T} \normtwo{\Sigma_{w_\tau}}$.
\end{algorithmic}
\end{algorithm}

We note that finding the unit vector $u_{\tau}$ required in the for loop of Algorithm~\ref{alg:spectral} can be done in time $O(N d \log(d) / \eps)$.
Given the PSD matrix $A = \Sigma(w_\tau)$, we want to find a unit vector $u \in \R^d$ such that $u^\top A u \ge (1-\eps) \max_v (v^\top A v)$.  This is the (approximate) largest eigenvector problem which can be solved via power method in $O(\log(d) / \eps)$ iterations.
Since the matrix-vector multiplication $Av = \Sigma_{w_\tau} v = \left(X \diag(w_\tau) X^\top - X w_\tau w_\tau^\top X^\top\right) v$ can be computed in time $O(N d)$, the running time for finding such a vector $u_\tau$ is $O(N d \log(d) / \eps)$.

\medskip

The main result of this section is the following theorem:

\begin{theorem}[Gradient descent finds a good solution]
\label{thm:pgd}
Let $S$ be an $\eps$-corrupted set of $N = \tilde{\Omega}(d / \eps^2)$ samples 
from a $d$-dimensional Gaussian $\NN(\mus, I)$ with unknown mean $\mus$.
Suppose $S$ satisfies Condition~\eqref{eqn:good-moments-remaining} and Lemma~\ref{thm:dkklms}.
Then, after $\tilde O(N^2 d^4)$ iterations, Algorithm~\ref{alg:spectral} outputs a weight vector $w \in \R^N$ such that $\normtwo{\mu_w - \mus} = O(\eps \sqrt{\log(1/\eps)})$.
\end{theorem}

We first give a high-level overview of the proof. 
Our proof of Theorem~\ref{thm:pgd} can be divided into two steps:  
\begin{enumerate}
\item The first step is an immediate consequence of Theorem \ref{lem:moreau-good}, 
which allows us to conclude that any approximate stationary point 
(in the sense of Definition \ref{approxstationarity}) has good approximation guarantees. 
\item To finalize the proof, in the second step we show that simple iterative procedures such as (sub)gradient descent 
can converge in a polynomial number of iterations to such an approximate stationary point. We prove such a result by 
utilizing a simple and well-known observation: a minimax optimization problem which is smooth in the minimization 
parameter is weakly convex (after maximization) in the minimization parameter. This connection allows us to leverage 
recent literature \cite{drusvyatskiy2017proximal, davis2018stochastic} that provides convergence 
guarantees for weakly convex optimization problems to prove our algorithm finds an approximate stationary 
point in a polynomial number of iterations.
\end{enumerate}

To elaborate further, in the second step of our proof, we utilize and slightly 
generalize\footnote{The generalization is to deal with constraints and handle the fact 
that the inner maximization is not solved precisely.} the analysis of \cite{davis2018stochastic} 
and prove that projected sub-gradient descent can find an approximate stationary point. 
\begin{lemma}
\label{lem:minimax-pgd}
Let $\KK$ be a closed convex set.
Let $F(w, u)$ be a function which is $L$-Lipschitz and $\beta$-smooth with respect to $w$.
Consider the following optimization problem $\min_{w \in \KK} \max_{\normtwo{u}=1} F(w, u)$.

Starting from any initial point $w_0 \in \KK$, we run iterative updates of the form:
\begin{align*}
&\text{find } u_\tau \text{ with } F(w_\tau, u_\tau) \ge (1-\eps') \max_{u} F(w_\tau, u_\tau) \\
&w_{\tau+1} = \PP_{\KK}(w_\tau - \eta \nabla_w F(w_\tau, u_\tau)
\end{align*}
for $T$ iterations with step size $\eta = \frac{\gamma}{\sqrt{T}}$. Then, we have 
\begin{align*}
&\min_{0 \le \tau < T} \normtwo{\nabla f_\beta(w_\tau)}^2 
  \le \frac{2}{\sqrt{T}} \left(\frac{f_\beta(w_0) - \min_w f(w)}{\gamma} + \gamma \beta L^2 \right) + 4 \beta \eps'
\end{align*}
where $f_\beta(w)$ is the Moreau envelope as in Definition~\ref{def:moreau}.
\end{lemma}
As shown in Appendix~\ref{apx:algo-proof}, $F(w,u)$ associated with $f(w)$ obeys 
the required Lipschitz and smoothness property, with $L=\tilde O(\sqrt{N} d)$ and $\beta=\tilde O(N d)$.
In addition, we have $0 \le f(w) \le \tilde O(d)$ for all $w \in \Delta_{N,2\eps}$.
Thus, we can apply the result above with the constraint $\mathcal{K}=\Delta_{N,2\eps}$.
Theorem~\ref{thm:pgd} follows by combining Theorem~\ref{lem:moreau-good} and Lemma~\ref{lem:minimax-pgd}.
We defer the proofs to Appendix~\ref{apx:algo-proof}.

\makeatletter{}
\section{Discussion} \label{sec:conc}

The main conceptual contribution of this work is to establish 
an intriguing connection between algorithmic high-dimensional
robust statistics and non-convex optimization. Specifically, we showed
that high-dimensional robust mean estimation can be efficiently 
solved by directly applying a first-order method to a natural non-convex
formulation of the problem. 

The main technical contribution of this paper is in showing
that any approximate stationary point of our non-convex objective suffices
to solve the underlying learning problem. Our novel structural result may be viewed as 
an explanation as to why robust mean estimation can be solved
efficiently in high dimensions, despite its non-convexity. Specifically, we establish that the optimization
landscape of our non-convex objective is well-behaved, in a precise sense.

There are a number of directions along which our results could be improved.
At the technical level, it would be interesting to obtain faster convergence rates
for gradient descent (or other first-order methods), with linear convergence 
as the ultimate goal. We note that our upper bound is fairly loose and we did not make an explicit
effort to optimize the polynomial dependence. 

A natural direction is to extend our approach to more general robust estimation tasks, 
including covariance estimation~\cite{DKKLMS16, CDGW19}, 
sparse PCA~\cite{BDLS17, DKKPS19-sparse}, and robust regression~\cite{KlivansKM18, DKS19-lr}.
Such generalizations will appear in a followup work.

\section{Acknowledgments}
We thank Jelena Diakonikolas for sharing her expertise in optimization.

\bibliography{allrefs}
\bibliographystyle{alpha}


\appendix

\makeatletter{}
\section{Setting Constants in Section~\ref{sec:struc}}
\label{apx:const}
In this section, we describe how to appropriately set the universal constants $c_1, \ldots, c_4 \ge 1$ in Section~\ref{sec:struc}.
These constants are set in the following order: $c_1, c_3, c_4, c_2$.
In this order, each $c_i$ only depends on the constants set before it, and there is only a lower bound requirement on the value of each $c_i$ so we can set $c_i$ to a sufficiently large constant.

The constant $c_1$ appears in Condition~\eqref{eqn:good-moments-remaining}.
and is related to the constants involved in the concentration inequalities required to establish this condition.
With the right sample complexity, Condition~\eqref{eqn:good-moments-remaining} holds 
with high probability for $\delta = c_1 \eps \ln(1/\eps)$.

For the remaining three constants, recall that by assumption 
$r = \normtwo{\mu_w - \mus} \ge c_2 \eps \sqrt{\ln(1/\eps)} \ge \eps \sqrt{\ln(1/\eps)}$.

Next we choose $c_3$ such that $c_3 \ge 5 c_1$.
This is to guarantee that, in the proof of Lemma~\ref{lem:good-small-gr}, we have
$2 c_1 \ln(1/\eps) + 2 \sqrt{2 c_1 \ln(1/\eps)} \cdot r \le c_3 \cdot \frac{r^2}{\eps^2}$.

The constant $c_4$ appears in the proof of Lemma~\ref{lem:bad-large-gr}.
There are two inequalities related to $c_4$.
We need $c_4 \ge 50 c_1$ so that $\frac{c_4 r^2}{\eps} - c_1 \eps \ln(1/\eps) \ge 0.98 \cdot c_4 \cdot \frac{r^2}{\eps}$,
and we require $c_4 \ge \max(100, 6 c_3)$ so that $\frac{0.49 \cdot c_4 \cdot r^2}{\eps^2} - \frac{1.4 \cdot \sqrt{c_4} \cdot r^2}{\eps} > 2 c_3 \cdot \frac{r^2}{\eps^2}$.

Finally, we set the value of $c_2$, which appears in our final guarantee: 
we show that any stationary point $w$ of $f(w)$ satisfies $\normtwo{\mu_w - \mus} \le c_2 \eps \sqrt{\ln(1/\eps)}$.
The constant $c_2$ only depends on $c_4$.
At the beginning of the proof of Lemma~\ref{lem:bad-large-gr}, we need that if $\normtwo{\mu_w - \mus} \ge c_2 \eps \sqrt{\ln(1/\eps)}$, then $\normtwo{\Sigma_w} \ge 1 + c_4 \cdot \frac{r^2}{\eps}$.
By Lemma~\ref{thm:dkklms} from~\cite{DKKLMS16}, we know that this is possible if we set $c_2$ to be sufficiently large.

\section{Missing Proofs from Section~\ref{sec:algo}}
\label{apx:algo-proof}
In this section, we prove Theorem~\ref{lem:moreau-good} and Lemma~\ref{lem:minimax-pgd} from Section~\ref{sec:algo}.
These two statements play an important role in showing that projected sub-gradient descent efficiently finds 
an approximate stationary point $w$, and that $w$ is a good solution to our robust mean estimation task.

We briefly recall our notation. We use $X \in \R^{d \times N}$ to denote the sample matrix, 
$\Sigma_w = \left(X \diag(w) X^\top - X w w^\top X^\top\right)$,
$F(w, u) = u^\top \Sigma_w u$, $f(w) = \max_u F(w, u) = \normtwo{\Sigma_w}$, 
and $\Delta_{N,\eps} = \left\{ w \in \R^N : \normone{w}=1 \text{ and } 0 \le w_i \le \tfrac{1}{(1-\eps) N} \forall i \right\}$.

Note that we can assume without loss of generality that no input samples have very large $\ell_2$-norm.
This is because we can perform a standard preprocessing step that centers the input samples at the coordinate-wise median, 
which does not affect our mean estimation task.
We can then throw away all samples that are $\Omega(\sqrt{d \log d})$ far from the coordinate-wise median.
With high probability, the coordinate-wise median of all good samples are $O(\sqrt{d \log d})$ far from the true mean.
Assuming this happens, then no good samples are thrown away 
and the remaining samples satisfy $\max_i \normtwo{X_i} = O(\sqrt{d \log d})$.
Consequently, we have $\normtwo{\mu_w} = O(\sqrt{d \log d})$ for any $w \in \Delta_{N, \eps}$.

In Lemma~\ref{lem:minimax-L-beta}, we show that the function $F(w, u) = u^\top \Sigma_w u$ is Lipschitz and smooth 
with respect to $w$. 

\begin{lemma}
\label{lem:minimax-L-beta}
The function $F(w, u)$ is $L$-Lipschitz and $\beta$-smooth for $L = \tilde O(\sqrt{N} d)$ and $\beta = \tilde O(N d)$.
That is,
\begin{align*}
\abs{F(w, u) - F(\tilde w, u)} \le L \normtwo{\tilde w - w} \quad &\text{ for all } \; w, \tilde w, \in \Delta_{N,2\eps} \; \text{ and all unit vectors } \; u \in \R^d \\
\normtwo{\nabla_w F(w, u) - \nabla_w F(\tilde w, u)} \le \beta \normtwo{\tilde w - w} \quad &\text{ for all } \; w, \tilde w, \in \Delta_{N,2\eps} \; \text{ and all unit vectors } \; u \in \R^d \; .
\end{align*}
\end{lemma}
\begin{proof}
We use the $\ell_2$-norm of the gradient to bound $L$ from above.
We have
\begin{align*}
\normtwo{\nabla_w F(w, u)}
&= \normtwo{X^\top u \odot X^\top u - 2 (u^\top X w) X^\top u} \\
&\le \sqrt{N} \max_i (X_i^\top u)^2 + 2 \norminf{u^\top X} \normone{w} \normtwo{X} \normtwo{u} \\
&\le \sqrt{N} \max_i \normtwo{X_i}^2 + 2 \max_i \normtwo{X_i} \normtwo{X} \; .
\end{align*}
To bound from above the smoothness parameter, we have
\[
\normtwo{\nabla_w F(w, u) - \nabla_w F(\tilde w, u)} = 2 \abs{u^\top X (w - \tilde w)} \normtwo{X^\top u} \le 2 \normtwo{X}^2 \normtwo{w - \tilde w} \; .
\]
We conclude the proof by observing that, after the preprocessing step, we have $\max_i \normtwo{X_i} = O(\sqrt{d \log d})$ and consequently $\normtwo{X} = O(\sqrt{N d \log d})$.
Therefore, $L = O(\sqrt{N} d \log d)$ and $\beta = O(N d \log d)$.
\end{proof}

Recall that the Moreau envelope $f_\beta(w)$ is defined as
\[
f_\beta(w) = \min_{\tilde w}\II_{\KK}(\tilde w) + F(\tilde w) + \beta \normtwo{\tilde w - w}^2 = \min_{\tilde w \in \KK} f(\tilde w) + \beta \normtwo{\tilde w - w}^2 \;,
\]
where $\II_{\KK}(\cdot)$ is the support function of $\KK$.

We restate Theorem~\ref{lem:moreau-good} before proving it.

{\noindent \bf Theorem~\ref{lem:moreau-good}.~}
{\em
Consider the spectral norm loss $f(w)=\|\Sigma_w\|_2$ with $f_\beta(w)$ denoting 
the corresponding Moreau envelope function per Definition~\ref{def:moreau} with $\beta=2\|X\|_2^2$. 
Then, for any $w \in \Delta_{N,2\eps}$ obeying
\[
\normtwo{\nabla f_\beta(w)} = O(\log(1/\eps)),
\]
we have $\normtwo{\mu_w - \mus} = O(\eps \sqrt{\log(1/\eps)})$.
}
\begin{proof}
Let $\delta = \frac{c_3 c_2^2 \ln(1/\eps)}{\sqrt{2}}$, 
where $c_2$ and $c_3$ are the positive universal constants from Lemma~\ref{lem:bad-large-gr}.
We show that any $w \in \Delta_{N,2\eps}$ obeying $\normtwo{\nabla f_\beta(w)} \le \delta$ 
must satisfy that $\normtwo{\mu_w - \mus} \le O(\eps \sqrt{\log(1/\eps)})$.

The condition $\normtwo{\nabla f_\beta(w)} \le \delta$ implies that there exists a vector $\hat w$ 
such that (see, e.g.,~\citet{rockafellar2015convex}):
\[
\normtwo{\hat w - w} = \frac{\delta}{2\beta} \quad \text{ and }\quad \min_{g \in \partial f(\hat w) + \partial \II_{\KK}(\hat w)} \normtwo{g} \le \delta \;.
\]
We first show that $\hat w$ is a good solution.

It is well known that the subdifferential of the support function is the normal cone, which is in turn the polar of the tangent cone.
That is,
\[
\partial \II_{\KK}(\hat w) = \NN_{\KK}(\hat w) = (\CC_{\KK}(\hat w))^\circ \; .
\]
Thus, there exists a vector $g = \nu + v$ with $\normtwo{g} \le \delta$ such that 
$\nu \in \partial f(\hat w)$ and $v \in (\CC_{\KK}(\hat w))^\circ$.
Now consider any unit vector $u \in \CC_{\KK}(\hat w)$:
\[
-\delta \le u^\top g = u^\top \nu + u^\top v \le u^\top \nu \;,
\]
where the last step follows from the definition of the polar set.
In other words, there exists a vector $\nu \in \partial f(\hat w)$ such that
\begin{equation}
\label{eqn:u-subgradient}
- \nu^\top u \le \delta \quad \text{ for all unit vectors } \; u \in \CC_{\KK}(\hat w) \;.
\end{equation}

Suppose $\normtwo{\mu_{\hat w} - \mus} \ge c_2 \eps \sqrt{\ln(1/\eps)}$. 
Then for the $v \in \partial f(\hat w)$ in question, we can use Lemmas~\ref{lem:bad-large-gr}~and~\ref{lem:good-small-gr} 
to find two coordinates $i$ and $j$ such that
\[
\hat w_i > 0, \; \hat w_j < \frac{1}{(1-2\eps)N}, \; \text{ and } \; \nu_i - \nu_j > c_3 \frac{\normtwo{\mu_{\hat w} - \mus}^2}{\eps^2} \ge c_3 c_2^2 \ln(1/\eps) = \sqrt{2} \delta \;.
\]
However, this contradicts Condition~\eqref{eqn:u-subgradient}, 
because for the unit vector $u = \frac{1}{\sqrt{2}}(e_j - e_i)$, where $e_i$ is the $i$-th basis vector,
we have $u \in \CC_{\Delta_{N,2\eps}}(\hat w)$ but
\[
- \nu^\top u = \frac{\nu_i - \nu_j}{\sqrt{2}} > \delta \;.
\]
Therefore, $\hat w$ must satisfy $\normtwo{\mu_{\hat w} - \mus} < c_2 \eps \sqrt{\ln(1/\eps)}$.

We conclude the proof by noticing that $w$ is very close to $\hat w$, 
so if $\hat w$ is a good solution, then $w$ must also be a good solution:
\begin{align*}
\normtwo{\mu_w - \mus}
&\le \normtwo{\mu_w - \mu_{\hat w}} + \normtwo{\mu_{\hat w} - \mus} \\
&\le \normtwo{X} \normtwo{w - \hat w} + c_2 \eps \sqrt{\ln(1/\eps)} \\
&= O(\beta^{-1/2} \delta + \eps \sqrt{\log(1/\eps)}) = O(\eps \sqrt{\log(1/\eps)}) \; .
\end{align*}
In the last two steps, we used the fact that $\normtwo{\hat w - w} = \frac{\delta}{2\beta}$ 
and $\beta = 2 \normtwo{X}^2$ (see Lemma~\ref{lem:minimax-L-beta}). 
This completes the proof of Theorem~\ref{lem:moreau-good}.
\end{proof}

We restate Lemma~\ref{lem:minimax-pgd} before proving it.
We note that the proof of Lemma~\ref{lem:minimax-pgd} is directly inspired by the proof of Theorem 2.1 in~\cite{davis2018stochastic}.

{\noindent \bf Lemma~\ref{lem:minimax-pgd}.~}
{\em
Let $\KK$ be a closed convex set.
Let $F(w, u)$ be a function which is $L$-Lipschitz and $\beta$-smooth with respect to $w$.
Consider the following optimization problem $\min_{w \in \KK} \max_{\normtwo{u}=1} F(w, u)$.

Starting from any initial point $w_0 \in \KK$, we run iterative updates of the form:
\begin{align*}
&\text{Find } u_\tau \text{ with } F(w_\tau, u_\tau) \ge (1-\eps') \max_{u} F(w_\tau, u_\tau) ; \\
&w_{\tau+1} = \PP_{\KK}(w_\tau - \eta \nabla_w F(w_\tau, u_\tau) \;,
\end{align*}
for $T$ iterations with step size $\eta = \frac{\gamma}{\sqrt{T}}$. Then, we have 
\begin{align*}
&\min_{0 \le \tau < T} \normtwo{\nabla f_\beta(w_\tau)}^2 \\
&\quad \le \frac{2}{\sqrt{T}} \left(\frac{f_\beta(w_0) - \min_w f(w)}{\gamma} + \gamma \beta L^2 \right) + 4 \beta \eps' \;,
\end{align*}
where $f_\beta(w)$ is the Moreau envelope, as in Definition~\ref{def:moreau}.
}
\begin{proof}
Note that since $f$ is $\beta$-smooth with respect to $w$ and $u_\tau$ is an approximate maximizer 
for $w_\tau$, for any $\tilde w \in \KK$, we have that
\begin{align}
f(\tilde w) \ge F(\tilde w, u_\tau)
&\ge F(w_\tau, u_\tau) + (\nabla_w F(w_\tau, u_\tau))^\top (\tilde w - w_\tau) - \frac{\beta}{2} \normtwo{\tilde w - w_\tau}^2 \nonumber \\
&\ge f(w_\tau) - \eps' + (\nabla_w F(w_\tau, u_\tau))^\top (\tilde w - w_\tau) - \frac{\beta}{2} \normtwo{\tilde w - w_\tau}^2  \;. \label{eqn:tmp-moreau}
\end{align}
To continue, define the proximal function
\[
\prox_{f_\beta}(w) = \arg\min_{\tilde w \in \KK} \left(f(\tilde w) + \beta \normtwo{\tilde w - w} \right) \;, 
\]
and let $\hat w_\tau = \prox_{f_{\beta}}(w_\tau)$.

Now we have
\begin{align*}
f_\beta(w_{\tau+1})
&\le f(\hat w_\tau) + \beta \normtwo{\hat w_\tau - w_{\tau + 1}} \\
&= f(\hat w_\tau) + \beta \normtwo{\hat w_\tau - \Pi_{\KK}(w_\tau - \eta \nabla_w F(w_\tau, u_\tau))} \\
&\le f(\hat w_\tau) + \beta \normtwo{\hat w_\tau - w_\tau + \eta \nabla_w F(w_\tau, u_\tau)} \tag{convexity of $\KK$} \\
&= f(\hat w_\tau) + \beta \normtwo{\hat w_\tau - w_\tau}^2 + 2 \eta \beta (\nabla_w F(w_\tau, u_\tau))^\top (\hat w_\tau - w_\tau) + \eta^2 \beta \normtwo{\nabla_w F(w_\tau, u_\tau)}^2 \\
&= f_\beta(w_\tau) + 2 \eta \beta (\nabla_w F(w_\tau, u_\tau))^\top (\hat w_\tau - w_\tau) + \eta^2 \beta \normtwo{\nabla_w F(w_\tau, u_\tau)}^2 \tag{$\hat w_\tau = \prox_{f_{\beta}}(w_\tau)$} \\
&\le f_\beta(w_\tau) + 2 \eta \beta (\nabla_w F(w_\tau, u_\tau))^\top (\hat w_\tau - w_\tau) + \eta^2 \beta L^2 \tag{$F(w, u)$ is $L$-Lipschitz in $w$} \\
&\le f_\beta(w_\tau) + 2 \eta \beta \left(f(\hat w_\tau) - f(w_\tau) + \eps' + \frac{\beta}{2}\normtwo{\hat w_\tau - w_\tau}^2 \right) + \eta^2 \beta L^2 \; .\tag{by Inequality~\eqref{eqn:tmp-moreau}}
\end{align*}

Summing the above over $\tau$, 
we obtain
\[
f_\beta(w_T) \le f_{\beta}(w_0) + 2\eta\beta\sum_{\tau=0}^{T-1} \left(f(\hat w_\tau) - f(w_\tau) + \frac{\beta}{2}\normtwo{\hat w_\tau - w_\tau}^2 \right) + \eta^2 \beta L^2 T + 2 \eta \beta T \eps' \;.
\]
Dividing by $2 \eta \beta T$, we get
\begin{align*}
\frac{1}{T} \sum_{\tau=0}^{T-1} \left(f(w_\tau) - f(\hat w_\tau) - \frac{\beta}{2}\normtwo{\hat w_\tau - w_\tau}^2 \right)
&\le \frac{f_\beta(w_0) - f_\beta(w_T)}{2 \eta \beta T} + \frac{\eta L^2}{2} + \eps' \\
&\le \frac{f_\beta(w_0) - \min_w f(w)}{2 \eta \beta T} + \frac{\eta L^2}{2} + \eps' \;.
\end{align*}
Observe that the function $w \to f(w) + \beta \normtwo{w - w_\tau}^2$ is $\beta$-strongly convex, therefore
\begin{align*}
&\quad f(w_\tau) - f(\hat w_\tau) - \frac{\beta}{2}\normtwo{\hat w_\tau - w_\tau}^2 \\
&= \left(f(w_\tau) + \beta \normtwo{w_\tau - w_\tau}^2 \right) - \left( f(\hat w_\tau) + \beta \normtwo{w_\tau - \hat w_\tau}^2 \right) + \frac{\beta}{2}\normtwo{w_\tau - \hat w_\tau}^2 \\
&\ge \frac{\beta}{2} \normtwo{\hat w_\tau - w_\tau}^2 + \frac{\beta}{2} \normtwo{\hat w_\tau - w_\tau}^2 \tag{strong convexity} \\
&= \beta \normtwo{\hat w_\tau - w_\tau}^2 = \frac{1}{4\beta} \normtwo{\nabla f_{\beta}(w_\tau)}^2 \; .
\end{align*}
In the above, we used the fact that for a $\beta$-strongly convex function $h(w) = \II_{\KK}(w) + f(w) + \beta \normtwo{w - w_\tau}^2$, we have $g(w_\tau) - g(\hat w_\tau) \ge \frac{\beta}{2} \normtwo{w_\tau - \hat w_\tau}^2$.

Combining the two inequalities above, we arrive at
\[
\frac{1}{T} \sum_{\tau=0}^{T-1} \normtwo{\nabla f_{\beta}(w_\tau)}^2 \le 
  2 \frac{f_\beta(w_0) - \min_w f(w)}{\eta T} + 2 \eta \beta L^2 + 4 \beta \eps' \; .
\]
Finally, setting the step size $\eta = \frac{\gamma}{\sqrt{T}}$, we conclude that
\[
\min_{0 \le \tau < T} \normtwo{\nabla f_{\beta}(w_\tau)}^2 \le \frac{2}{\sqrt{T}} \left(\frac{f_\beta(w_0) - \min_w f(w)}{\gamma} + \gamma \beta L^2 \right) + 4 \beta \eps' \;. 
\]
This completes the proof of Lemma~\ref{lem:minimax-pgd}. \qedhere
\end{proof}

\section{Minimizing Softmax of Spectral Norm}
\label{sec:softmax}
In this section, we analyze our alternate non-convex formulation that replaces the spectral 
norm with a softmax. Note that when the largest eigenvalue of $\Sigma_w$ is not unique, 
the spectral norm of $\Sigma_w$ may not be differentiable with respect to $w$.
Instead of considering sub-gradients, we can minimize the softmax of the eigenvalues of $\Sigma_w$, 
which is a smoothed version of spectral norm that is differentiable everywhere.

Formally, we minimize the following non-convex objective function:
\begin{align}
\label{eqn:f-smax}
f(w) &= \smax_{\rho}\left(\Sigma_w\right) = \frac{1}{\rho}\ln \tr(\exp(\rho \Sigma_w)) \qquad \text{for} \qquad \rho = \frac{\ln d}{\eps} \; ,
\end{align}
where $X \in \R^{d \times N}$ is the sample matrix, and $\Sigma_w = X \diag(w) X^\top - X w w^\top X^\top$ is the weighted empirical covariance matrix.

The structure of this section is as follows: 
In Section~\ref{sec:prelim-smax}, we start by recording some useful properties of the softmax objective.
In Section~\ref{sec:struc-smax}, we prove our key structural result for this section (Theorem~\ref{thm:no-bad-local-opt-smax}), 
establishing that any approximate stationary point $w$ of $f(w)$ provides a good estimate $\mu_w$ of the true mean $\mus$ .
In Section~\ref{sec:algo-smax}, we present our algorithmic result (Theorem~\ref{thm:final-softmax}), 
which states that we can efficiently find an approximate stationary point of $f(w)$ via projected gradient descent.

\subsection{Basic Properties of Softmax} \label{sec:prelim-smax}
\begin{lemma}[Duality of softmax]
\label{lem:smax-mat-dual}
For any $Z \in \R^{n\times n}$ and $\rho>0$, let $\smax_\rho(Z) \eqdef \frac{1}{\rho}\ln \tr(\exp(\rho Z))$.
We have the following identity
\[
\smax_\rho(Z)  = \max_{Y \in \Delta_{n\times n}} \left(Y \bullet Z - \frac{1}{\rho} Y \bullet \log Y \right) \;.
\]
\end{lemma}
\begin{proof}
Fix $Z \in \R^{n \times n}$.
Let $f(Y) = Y\bullet Z - \frac{1}{\rho} Y\bullet \log Y$.
Using the KKT conditions, we know that when $f(Y)$ is maximized, 
we have $\frac{\partial f}{\partial Y} = \lambda I$, for some $\lambda \in \R$.
Combining this with $\frac{\partial f}{\partial Y} = Z - \frac{1}{\rho}(\log Y + I)$, it follows that $f(Y)$ is maximized at
\[
Y^\star = \exp(\rho Z - (\rho \lambda + 1) I) = \frac{\exp(\rho Z)}{\tr (\exp(\rho Z))} \;,
\]
where the second equality holds because $Y^\star \in \Delta_{n \times n}$.
One can substitute $Y^\star$ into the definition of $f(Y)$ and verify that $f(Y^\star) = \smax_\rho(Z)$.
\end{proof}

\begin{corollary}[Softmax and max] \label{cor:smax-max}
For any PSD matrix $Z \in \R^{n\times n}$ and $\rho>0$, we have that
$\lambda_{\max}(Z) \le \smax_\rho(Z) \le \lambda_{\max}(Z) + \frac{\ln n}{\rho}$.
Moreover, for $Y = \frac{\exp(\rho Z)}{\tr (\exp(\rho Z))}$, we have that $Y \bullet Z \ge \smax_\rho(Z) - \frac{\ln n}{\rho}$.
\end{corollary}
\begin{proof}
Observe that
\[
\smax_\rho(Z) = \frac{1}{\rho}\ln \tr(\exp(\rho Z)) \ge \frac{1}{\rho}\ln \lambda_{\max}(\exp(\rho Z)) = \lambda_{\max}(Z) \; ,
\]
and
\[
\smax_\rho(Z) = \frac{1}{\rho}\ln \tr(\exp(\rho Z)) \le \frac{1}{\rho}\ln (n \cdot \lambda_{\max}(\exp(\rho Z))) = \lambda_{\max}(Z) + \frac{\ln n}{\rho} \; .
\]

For the second claim, by Lemma~\ref{lem:smax-mat-dual}, we know that $\smax_\rho(Z) = Y \bullet Z - \frac{1}{\rho}Y \bullet \log Y$.
The claim then follows from the fact that $Y \bullet \log Y \ge -\ln n$ for all $Y \in \Delta_{n \times n}$.
\end{proof}

When working with the matrix exponentials in our softmax objective function $f$, the following chain rule formula will be useful 
to compute the Hessian of $f$ (see, e.g.,~\cite{Wilcox1967exp}).

\begin{lemma}[Derivative of matrix exponential]
\label{lem:d-mat-exp}
For a symmetric matrix function $X(t)$ that depends on a scalar $t$, we have that
\[
\frac{d}{d t} \exp(X(t)) = \int_{0}^1 \exp(\alpha X(t)) \frac{d X(t)}{d t} \exp((1-\alpha) X(t)) d \alpha\; .
\]
\end{lemma}

\subsection{Structural Result: Any Approximate Stationary Point Suffices}
\label{sec:struc-smax}
The gradient of our softmax objective function is
\begin{align}
\label{eqn:nabla-f}
\nabla f(w) &= \diag(X^\top Y X) - 2 X^\top Y X w \;,
  \quad \text{where} \quad
  Y = \frac{\exp(\rho \Sigma_w)}{\tr (\exp (\rho \Sigma_w))} \; .
\end{align}

Notice that $Y \in \Delta_{N \times N}$ is a convex combination of directions.
That is, we can write $Y = \sum_{k=1}^d \lambda_k u_k u_k^\top$, where $u_k \in \R^d$ and $\sum_k \lambda_k = 1$.
The gradient $\nabla f(w)$ is the same as the gradient of $w$ for the one-dimensional problem, 
where the input samples are $(X_i^\top Y^{1/2})_{i=1}^N$. 
Equivalently, $\nabla f(w)$ tries to move $w$ towards minimizing the average variance
\[
\sum_k \lambda_k \left(\sum_{i} w_i (X_i^\top u_k)^2 - \left(\sum_{i} w_i (X_i^\top u_k)\right)^2 \right)
\]
of the projections of $X$ along the directions $\{u_k\}$.

The intuition is as follows:
The goal is to show that $\lambda_{\max}(\Sigma_w)$ is small at any stationary point $w$ of $\smax_\rho(\Sigma_w)$. 
Now fix some $w \in \Delta_{N,2\eps}$, where $\lambda_{\max}(\Sigma_w)$ is large. 
Then $\smax_\rho(\Sigma_w)$ must be large.
By the duality of softmax, there is a combination of directions $Y$ such that: 
(1) the one-dimensional samples $(X_i^\top Y^{1/2})_{i=1}^N$ weighted by $w$ have large variance, 
and (2) the derivative of $\smax_\rho(\Sigma_w)$ is the same as the derivative 
for minimizing variance on this one-dimensional instance.
We proceed by examining this one-dimensional instance, which is easier to analyze.
We show that $w$ cannot be a stationary point, because we can always reduce 
the variance by increasing the weight on one of the good samples 
and reducing the weight on one of the bad samples.

Formally, we use the following notion of approximate stationarity for our constrained 
non-convex minimization problem.

\begin{definition}
\label{def:stationary-smax}
Fix a convex set $\KK$.
For $\delta>0$, we say $x \in \KK$ is a $\delta$-stationary point of $f$ if the following condition holds:
For any unit vector $u$ where $x + \alpha u \in \KK$ for some $\alpha > 0$, we have $u^\top \nabla f(x) \ge - \delta$.
\end{definition}

Our main structural result in this section is the following theorem.

\begin{theorem}[Any stationary point of $f(w)$ is a good solution]
\label{thm:no-bad-local-opt-smax}
Let $S$ be an $\eps$-corrupted set of $N = \tilde \Omega(d/\eps^2)$ samples drawn 
from a $d$-dimensional Gaussian $\NN(\mus, I)$ with unknown mean $\mus$.
Suppose $S$ satisfies Condition~\eqref{eqn:good-moments-remaining} and Lemma~\ref{thm:dkklms}.

Let $f(w)$ be the softmax objective as defined in Equation~\eqref{eqn:f-smax}.
Let $\delta = c \ln(1/\eps)$ for some universal constant $c$.
For any $w \in \Delta_{N,2\eps}$ that is a $\delta$-stationary point of $f(w)$, 
we have $\normtwo{\mu_w - \mus} = O(\eps \sqrt{\log(1/\eps)})$.
\end{theorem}

Theorem~\ref{thm:no-bad-local-opt-smax} follows directly from 
Lemmas~\ref{lem:kkt-w-smax},~\ref{lem:bad-large-gr-smax},~and~\ref{lem:good-small-gr-smax}.

For the rest of this subsection, we assume the input samples satisfy Condition~\eqref{eqn:good-moments-remaining} and Lemma~\ref{thm:dkklms}, 
and we fix an approximate stationary point $w \in \Delta_{N, 2\eps}$ of the softmax objective.
We establish the following bimodal sub-gradient property which holds at all (approximate) stationary points.

\begin{lemma}[Bimodal sub-gradient property at stationary points]
\label{lem:kkt-w-smax}
Fix $w \in \Delta_{N,2\eps}$.
Let $S_- = \{i : w_i > 0\}$ and $S_+ = \{i : w_i < \frac{1}{(1-2\eps)N} \}$ denote the set of coordinates of $w$ 
that can decrease and increase respectively.
If $w$ is a $\delta$-stationary point of $f(w)$, then $\nabla f(w)_i \le \nabla f(w)_j + \sqrt{2} \delta$ for all $i \in S_-$ and $j \in S_+$.
\end{lemma}
\begin{proof}
Suppose there is some $i \in S_-$ and $j \in S_+$ such that $\nabla f(w)_i > \nabla f(w)_j + \sqrt{2} \delta$.

Consider the unit vector $u = \frac{1}{\sqrt{2}}(e_j - e_i)$, where $e_i$ is the $i$-th basis vector.
We have $w + \alpha u \in \Delta_{N,2\eps}$ for $\alpha = \min(w_i, \frac{1}{(1-2\eps)N} - w_j) > 0$, but
\[
u^\top \nabla f(x) = \frac{\nabla f(w)_j - \nabla f(w)_i}{\sqrt{2}} < - \delta \; ,
\]
which violates the assumption that $w$ is a $\delta$-approximate stationary point (Definition~\ref{def:stationary-smax}).
\end{proof}

At a high level, we prove Theorem~\ref{thm:no-bad-local-opt-smax} by showing that if $\mu_w$ is far from $\mus$, then $w$ violates Lemma~\ref{lem:kkt-w-smax}.
More specifically, if $\mu_w$ is far from $\mus$, then there exists a bad sample with index $j \in S_-$ whose gradient is large (Lemma~\ref{lem:bad-large-gr-smax}).
Meanwhile, the concentration bound in Condition~\eqref{eqn:good-moments-remaining} guarantees 
that there exists a good sample with index $i \in S_+$ whose gradient is small (Lemma~\ref{lem:good-small-gr-smax}).

We frequently use the partial derivative of $f(w)$ with respect to $w_i$ in our analysis:
\begin{align*}
\nabla f(w)_i
  &= X_i^\top Y X_i - 2 X_i^\top Y \mu_w \\
  &= (X_i - \mus)^\top Y (X_i - \mus) - 2(X_i - \mus)^\top Y (\mu_w - \mus) \\
    &\qquad + {\mus}^\top Y (\mus - 2\mu_w) \; .
\end{align*}
Notice that the last term in $\nabla f(w)_i$ is the same for all $i$.
Since our goal is to identify $i \in S_-$ and $j \in S_+$ such that $\nabla f(w)_i > \nabla f(w)_j$, 
we can focus on the first two terms.

We have the following lemmas:

\begin{lemma}\label{lem:bad-large-gr-smax}
Fix $w \in \Delta_{N,2\eps}$ and assume that Condition~\eqref{eqn:good-moments-remaining} and Lemma~\ref{thm:dkklms} hold.
Let $c_2$ and $c_3$ be universal constants.
Let $r = \normtwo{\mu_w - \mus}$ and suppose $r \ge c_2 \eps \sqrt{\ln(1/\eps)}$.
Then, there exists $i \in (B \cap S_-)$ such that
\[
\nabla f(w)_i - {\mus}^\top Y (\mus - 2\mu_w) > 2 c_3 \cdot \frac{r^2}{\eps^2} \; .
\]
\end{lemma}

\begin{lemma}
\label{lem:good-small-gr-smax}
Consider the same setting as in Lemma~\ref{lem:bad-large-gr-smax}.
There exists $j \in (G \cap S_+)$ such that
\[
\nabla f(w)_j - {\mus}^\top Y (\mus - 2\mu_w) \le c_3 \cdot \frac{r^2}{\eps^2} \; .
\]
\end{lemma}

We defer the proofs of Lemmas~\ref{lem:bad-large-gr-smax}~and~\ref{lem:good-small-gr-smax} to Section~\ref{sec:proof-gb-gr-smax}, 
and we first use them to prove Theorem~\ref{thm:no-bad-local-opt-smax}.

\begin{proof}[Proof of Theorem~\ref{thm:no-bad-local-opt-smax}]
Suppose that $w$ is a bad solution where $\normtwo{\mu_w - \mus} \ge c_2 \eps \sqrt{\ln(1/\eps)}$.
Since we assume Condition~\eqref{eqn:good-moments-remaining} and Lemma~\ref{thm:dkklms} both hold on the input samples, we can use Lemmas~\ref{lem:bad-large-gr-smax}~and~\ref{lem:good-small-gr-smax} to find two coordinates $i \in S_-$ and $j \in S_+$, such that the bimodal sub-gradient property in Lemma~\ref{lem:kkt-w-smax} does not hold at $w$.
Therefore, $w$ is not a $\delta$-approximate stationary point for some $\delta = \sqrt{2} c_3 \frac{\normtwo{\mu_w - \mus}^2}{\eps^2} \ge \sqrt{2} c_3 c_2^2 \ln(1/\eps)$, that is, we can set $c = \sqrt{2} c_3 c_2^2$.
\end{proof}

\subsubsection{Proofs of Lemmas~\ref{lem:bad-large-gr-smax}~and~\ref{lem:good-small-gr-smax}}
\label{sec:proof-gb-gr-smax}
In this section, we prove Lemmas~\ref{lem:bad-large-gr-smax}~and~\ref{lem:good-small-gr-smax}.

The proofs of these lemmas are conceptually similar to the proofs 
of related lemmas (Lemmas~\ref{lem:bad-large-gr}~and~\ref{lem:good-small-gr}) in Section~\ref{sec:struc}.
We include their proofs here to make this section self-contained.
The main difference is that we switch to the softmax objective, and consequently, we need to work with multiple directions simultaneously.
That is, we consider the projections using $Y$ instead of the projections along the maximum eigenvector of $\Sigma_w$.

Lemma~\ref{lem:bad-large-gr-smax} states that when $\mu_w$ is far from $\mus$, there exists an index $i \in (B \cap S_-)$ such that the gradient $\nabla f(w)_i$ is relatively large.

Recall that the gradient $\nabla f(w)$ in Equation~\eqref{eqn:nabla-f} 
is the same as the gradient of the variance (weighted by $w$) of the one-dimensional samples $\left(X_i^\top Y^{1/2}\right)_{i=1}^N$.
For this one-dimensional problem, a sample far from the (projected) true mean must have large gradient.
Our objective is to find such a sample for which we can decrease its weight.
More specifically, since $w$ is assumed to be a bad solution, 
and the softmax objective is close to the spectral norm of $\Sigma_w$, 
the weighted empirical variance of the projected samples is very large.
Because the good samples cannot have this much variance, most of the variance comes from the bad samples.
We prove that among these bad samples that contribute a lot to the variance, 
one of them must be very far from the (projected) true mean and hence has a large gradient, 
which satisfies Lemma~\ref{lem:bad-large-gr-smax}.

We use $c_1, \ldots, c_4$ to denote universal positive constants that are independent of $N$, $d$, and $\eps$.
These constants can be set in a way that is similar to that in Section~\ref{sec:struc} (see Appendix~\ref{apx:const}).
The universal constant $c$ in Theorem~\ref{thm:no-bad-local-opt-smax} can be set as $c = \sqrt{2} c_3 c_2^2$ after we set $c_2$ and $c_3$.

\begin{proof}[Proof of Lemma~\ref{lem:bad-large-gr-smax}]
We first show that $\Sigma_w \bullet Y$ is relatively large.
By Lemma~\ref{thm:dkklms}, we know that if $\normtwo{\mu_w - \mus} \ge r$ and $r \ge c_2 \eps \sqrt{\ln(1/\eps)}$, then
\[
\lambda_{\max}(\Sigma_w) \ge 1 + c_4 \cdot \frac{r^2}{\eps} \; .
\]
By Corollary~\ref{cor:smax-max}, for $Y = \frac{\exp(\rho \Sigma_w)}{\tr (\exp(\rho \Sigma_w))}$ 
and $\rho = \frac{\ln d}{\eps}$, we have
\[
\Sigma_w \bullet Y \ge \smax_\rho(\Sigma_w) - \eps \ge \lambda_{\max}(\Sigma_w) - \eps \ge 1 - \eps + \frac{c_4 r^2}{\eps} \; .
\]
Recall that $\Sigma_w = \sum_{i=1}^N w_i (X_i - \mu_w)(X_i - \mu_w^\top)$.
If we replace $\mu_w$ with $\mus$, we have
\[
\sum_{i=1}^N w_i (X_i - \mus)(X_i - \mus)^\top \, \succeq \, \Sigma_w \; ,
\]
and therefore,
\[
\left(\sum_{i=1}^N w_i (X_i - \mus)(X_i - \mus)^\top \right) \bullet Y \ge \Sigma_w \bullet Y \ge 1 - \eps + \frac{c_4 r^2}{\eps} \; .
\]

Next we show that most of the variance is due to bad samples.
By Condition~\eqref{eqn:good-moments-remaining},
\[
\left(\sum_{i \in G} w_i (X_i - \mus)(X_i - \mus)^\top\right) \bullet Y \le 1 + c_1 \cdot \eps \ln(1/\eps) \; .
\]
Consequently,
\[
\left(\sum_{i \in B} w_i (X_i - \mus)(X_i - \mus)^\top\right) \bullet Y \ge \frac{c_4 r^2}{\eps} - \eps - c_1 \eps \ln(1/\eps) \ge 0.98 \cdot c_4 \cdot \frac{r^2}{\eps} \; .
\]
The last step is because $r \ge c_2 \cdot \eps\sqrt{\ln(1/\eps)}$ and we can choose $c_2$ and $c_4$ to be sufficiently large.

At this point, we know that when $r = \normtwo{\mu_w - \mus}$ is large, most of the variance is due to the bad samples.
However, the total weight $w_B$ on the bad samples is at most $\eps N \cdot \frac{1}{(1-2\eps)N} \le 2 \eps$.
Therefore, there must be some $i \in B$ with $w_i > 0$ and
\[
\left((X_i - \mus)(X_i - \mus)^\top\right) \bullet Y \ge \frac{0.98 \cdot c_4 \cdot r^2 \cdot \eps^{-1}}{w_B} \ge 0.49 \cdot c_4 \cdot \frac{r^2}{\eps^2} \; .
\]
By definition, $i \in B \cap S_{-}$.
It remains to show that $\nabla f(w)_i$ is large.
\begin{align*}
\nabla f(w)_i - {\mus}^\top Y (\mus - 2\mu_w)
  &= \left((X_i - \mus)(X_i - \mus)^\top\right) \bullet Y - 2 \left((X_i - \mus)(\mu_w - \mus)^\top\right) \bullet Y \\
  &\ge \normtwo{Y^{1/2} (X_i - \mus)}^2 - 2 \normtwo{Y^{1/2} (X_i - \mus)} \cdot \normtwo{Y^{1/2}} \cdot \normtwo{\mu_w - \mus} \\
  &\ge \frac{0.49 \cdot c_4 \cdot r^2}{\eps^2} - 2 \cdot \frac{0.7 \cdot \sqrt{c_4} \cdot r}{\eps} \cdot 1 \cdot r \\
  &> 2 c_3 \cdot \frac{r^2}{\eps^2} \; .
\end{align*}
The first inequality is because $Y \in \Delta_{d \times d}$.
The last step uses the fact that $c_4$ can be sufficiently large.
This completes the proof of Lemma~\ref{lem:bad-large-gr-smax}.
\end{proof}

Lemma~\ref{lem:good-small-gr-smax} states that there exists an index $j \in (G \cap S_+)$ 
such that the gradient $\nabla f(w)_j$ is relatively small.
Similar to the proof of Lemma~\ref{lem:bad-large-gr-smax}, for the projected one-dimensional instance, 
a sample close to the (projected) true mean should have small gradient. 
Our goal is to find such a sample for which we can increase its weight.
Recall that $S^+$ contains the samples whose weight can be increased.
We first prove that there are at least $\eps N$ good samples in $S^+$.
Among these $\eps N$ good samples, the concentration bounds 
imply that there must exist some $X_j$ that is close to the (projected) true mean.
The derivative $\nabla f(w)_j$ satisfies Lemma~\ref{lem:good-small-gr-smax}.

\begin{proof}[Proof of Lemma~\ref{lem:good-small-gr-smax}]
Recall that $S^+$ contains every coordinate $i$ where $w_i < \frac{1}{(1-2\eps)N}$.
Since at most $(1-2\eps)N$ samples can have the maximum weight $\frac{1}{(1-2\eps)N}$, we know that $|S^+| \ge 2\eps N$.
Combining this with $|G| = (1-\eps)N$, we know that $|G \cap S^+| \ge \eps N$.

Fix a subset $G^+ \subseteq (G \cap S^+)$ of size $|G^+| = \eps N$.
We first show that, on average, samples in $G^+$ do not contribute much to the variance.

Let $w'$ be the uniform weight vector on $G$, i.e., $w'_i = \frac{1}{(1-\eps)N}$ for all $i \in G$ and $w'_i = 0$ otherwise.
Since $w' \in \Delta_{N,2\eps}$, by Condition~\eqref{eqn:good-moments-remaining}, we have that
\[
\normtwo{\sum_{i \in G} \frac{1}{|G|} (X_i - \mus)(X_i - \mus)^\top - I} \le c_1 \cdot \eps \ln(1/\eps) \;.
\]
Let $w''$ be the uniform weight vector on $S \setminus G^+ = (G \setminus G^+) \cup B$, 
i.e., $w''_i = \frac{1}{(1-\eps)N}$ for all $i \in ((G \setminus G^+) \cup B)$ and $w''_i = 0$ otherwise.
Since $w'' \in \Delta_{N,2\eps}$, again by Condition~\eqref{eqn:good-moments-remaining}, we have that
\[
\normtwo{\sum_{i \in G \setminus G^+} \frac{1}{|G|} (X_i - \mus)(X_i - \mus)^\top - I} \le c_1 \cdot \eps \ln(1/\eps) \;.
\]
Combining the previous two concentration bounds, we obtain that
\begin{align*}
\normtwo{\sum_{i \in G^+} \frac{1}{|G|} (X_i - \mus)(X_i - \mus)^\top}
  &\le \normtwo{\sum_{i \in G} \frac{1}{|G|} (X_i - \mus)(X_i - \mus)^\top - I} \\
  &\qquad + \normtwo{\sum_{i \in G \setminus G^+} \frac{1}{|G|} (X_i - \mus)(X_i - \mus)^\top - I}
  \le 2 c_1 \cdot \eps \ln(1/\eps) \; .
\end{align*}
As a result, because $Y \in \Delta_{d \times d}$, it follows that
\[
\left(\sum_{i \in G^+} \frac{1}{|G|} (X_i - \mus)(X_i - \mus)^\top \right) \bullet Y \le 2 c_1 \cdot \eps \ln(1/\eps) \;.
\]
Now we know that, on average, samples in $G^+$ do not contribute much to the variance.
We continue to show that one of these samples satisfies the lemma.

Let $j = \arg\min_{i \in G^+} \left(Y \bullet (X_i - \mus)(X_i - \mus)^\top \right)$.
We have that
\[
\left((X_j - \mus)(X_j - \mus)^\top \right) \bullet Y \le \frac{|G|}{|G^+|} \cdot 2 c_1 \cdot \eps \ln(1/\eps) \le 2 c_1 \ln(1/\eps) \;.
\]
Finally, because $(X_j - \mus)^\top Y (X_j - \mus) \le 2 c_1 \ln(1/\eps)$, 
we can bound $\nabla f(w)_j$ from above as follows:
\begin{align*}
\nabla f(w)_j - {\mus}^\top Y (\mus - 2\mu_w)
  &= \left((X_j - \mus)(X_j - \mus)^\top\right) \bullet Y - 2 \left((X_j - \mus)(\mu_w - \mus)^\top\right) \bullet Y \\
  &\le \normtwo{Y^{1/2} (X_j - \mus)}^2 + 2 \normtwo{Y^{1/2} (X_j - \mus)} \cdot \normtwo{Y^{1/2}} \cdot \normtwo{\mu_w - \mus} \\
  &\le 2 c_1 \ln(1/\eps) + 2 \sqrt{2 c_1 \ln(1/\eps)} \cdot 1 \cdot r \\
  &\le \frac{c_3}{2} \cdot \frac{r^2}{\eps^2} + \frac{c_3}{2} \cdot \frac{r}{\eps} \cdot r \le c_3 \cdot \frac{r^2}{\eps^2} \; .
\end{align*}
The last step uses that $c_3$ is sufficiently large, as well as the fact that $\ln(1/\eps) \le \frac{r^2}{\eps^2}$, 
because $r \ge c_2 \eps \sqrt{\ln(1/\eps)}$. This completes the proof of Lemma~\ref{lem:good-small-gr-smax}.
\end{proof}

\subsection{Convergence Rate of Minimizing Softmax} \label{sec:algo-smax}

\begin{algorithm}[h]
\caption{Robust Mean Estimation via Projected Gradient Descent on the Softmax Objective}
\label{alg:smax}
\begin{algorithmic}
  \STATE {\bf Input:} $\eps$-corrupted set of $N$ samples $\{X_i\}_{i=1}^N$ on $\R^d$ satisfying Condition~\eqref{eqn:good-moments-remaining}, and $\eps < \eps_0$.
  \STATE {\bf Output:} $w \in \R^N$ with $\normtwo{\mu_w - \mus} \le O(\eps \sqrt{\log(1/\eps)})$.
  \STATE Let $\rho = \ln d / \eps$.
  \STATE Let $\beta = \tilde O(N d^2 / \eps)$ be the smoothness parameter of the softmax objective $f(w) = \smax_\rho(\Sigma_w)$.
  \STATE Let $w_0$ be an arbitrary weight vector in $\Delta_{N,2\eps}$.
  \STATE Let $T = \tilde O(N d^3 / \eps)$ and $\eta = 1/\beta$.
  \FOR{$\tau=0$ {\bf to} $T-1$}
    \STATE $w_{\tau+1} = \PP_{\Delta_{N,2\eps}}\left(w_\tau - \eta \nabla f(w) \right)$, 
    where $\PP_\KK(\cdot)$ is the $\ell_2$-projection operator onto $\KK$.
  \ENDFOR \\
  \STATE {\bf return } $w_{\tau^\star}$ where $\tau^\star = \arg\min_{\, 0 \le \tau < T} \normtwo{w_{\tau+1} - w_\tau}$.
\end{algorithmic}
\end{algorithm}

In this section, we prove our algorithmic result for the softmax objective (Theorem~\ref{thm:final-softmax}).
We show that the projected gradient descent algorithm (Algorithm~\ref{alg:smax}) on $f$ 
can efficiently find an approximate stationary point $w$, and that $w$ is a good solution to our robust mean estimation task.

We first restate Theorem~\ref{thm:final-softmax} (correctness and iteration count of Algorithm~\ref{alg:smax}).

\medskip

{\noindent \bf Theorem~\ref{thm:final-softmax}.~}
{\em
Let $S$ be an $\eps$-corrupted set of $N = \tilde \Omega(d/\eps^2)$ samples drawn 
from a $d$-dimensional Gaussian $\NN(\mus, I)$ with unknown mean $\mus$.
Suppose $S$ satisfies Condition~\eqref{eqn:good-moments-remaining} and Lemma~\ref{thm:dkklms}.

Let $f(w)$ be the softmax objective as defined in Equation~\eqref{eqn:f-smax}.
After $\tilde O(N d^3 / \eps)$ iterations, projected gradient descent on $f(w)$ outputs a
point $w$ such that $\|\mu_{w} - \mus \|_2 = O(\eps\sqrt{\log(1/\eps)})$.}

\medskip

Theorem~\ref{thm:final-softmax} follows immediately from Lemmas~\ref{lem:pgd-rate-smax},~\ref{lem:smax-smoothness},~and~\ref{lem:smax-maxvalue}.

Lemma~\ref{lem:pgd-rate-smax} analyzes the convergence rate of (nonconvex) projected gradient descent.
The number of iterations in Lemma~\ref{lem:pgd-rate-smax} depends on the range and smoothness of the objective function.
Lemmas~\ref{lem:smax-smoothness}~and~\ref{lem:smax-maxvalue} upper bounds these two parameters for our softmax objective.

We note that Lemma~\ref{lem:pgd-rate-smax} appears to be folklore in the optimization literature, 
see, e.g.,~\cite{Beck-book}. For the sake of completeness, we provide a self-contained proof in the following subsection.

\begin{lemma} \label{lem:pgd-rate-smax}
Fix a (possibly non-convex) function $f$ and a convex set $\KK$.
Suppose $f$ is $\beta$-smooth on $\KK$ and $0 \le f(x) \le B$ for all $x \in \KK$.
If we run projected gradient descent with step size $\eta = \frac{1}{\beta}$ starting from an arbitrary $x_0 \in \KK$:
\[
x_{\tau+1} = \Pi_{\KK}\left(x_{\tau} - \eta \nabla f(x_\tau) \right) \;,
\]
where $\Pi_{\KK}$ is the projection onto $\KK$, 
we can compute a $\delta$-stationary point of $f$ in $O(\frac{\beta \cdot B}{\delta^2})$ iterations.
\end{lemma}

Recall that the softmax objective is
$f(w) = \smax_{\rho}\left(\Sigma_w\right) = \frac{1}{\rho}\ln \tr(\exp(\rho \Sigma_w))$ with $\rho = \frac{\ln d}{\eps}$.
A differentiable function $f$ is $\beta$-smooth on $\KK$ if $\normtwo{\nabla f(x) - \nabla f(y)} \le \beta \normtwo{x - y}$ for all $x, y \in \KK$.

\begin{lemma}[Smoothness of $f$]
\label{lem:smax-smoothness}
The softmax objective $f$ is $\beta$-smooth on $\Delta_{N,2\eps}$ for $\beta = \tilde O(N d^2 / \eps)$.
\end{lemma}

\begin{lemma}[Range of $f$]
\label{lem:smax-maxvalue}
The softmax objective $f$ satisfies that $0 \le f(w) \le \tilde O(d)$ for all $w \in \Delta_{N,2\eps}$.
\end{lemma}

We defer the proofs of Lemmas~\ref{lem:pgd-rate-smax},~\ref{lem:smax-smoothness},~and~\ref{lem:smax-maxvalue} 
to the next subsections and first use them to prove Theorem~\ref{thm:final-softmax}.

\begin{proof}[Proof of Theorem~\ref{thm:final-softmax}]
We first prove the correctness of Algorithm~\ref{alg:smax}.
Let $c$ be the universal constant in Theorem~\ref{thm:no-bad-local-opt-smax} and let $\delta = c \ln(1/\eps)$.
We run Algorithm~\ref{alg:smax} to obtain a $\delta$-stationary point $w$.
Since we assume the input samples satisfy Condition~\eqref{eqn:good-moments-remaining} and Lemma~\ref{thm:dkklms}, Theorem~\ref{thm:no-bad-local-opt-smax} states that $w$ is a good solution with $\normtwo{\mu_w - \mus} = O(\eps \sqrt{\ln(1/\eps)})$.

We now analyze the number of iterations $T$.
By Lemma~\ref{lem:pgd-rate-smax}, it is sufficient to set $T = O(\frac{\beta \cdot B}{\delta^2})$, as in Algorithm~\ref{alg:smax}.
Substituting the upper bounds on $\beta$ and $B$ from Lemmas~\ref{lem:smax-smoothness}~and~\ref{lem:smax-maxvalue}, 
and our choice of $\delta$, we get
\[
T = O(\beta \cdot B \cdot \delta^{-2}) = \tilde O(N d^2 / \eps) \cdot \tilde O(d) \cdot O(\log^{-2}(1/\eps)) = \tilde O(N d^3 / \eps) \;,
\]
as claimed.
\end{proof}

\subsection{Proof of Lemma~\ref{lem:pgd-rate-smax}}
In this section, we prove Lemma~\ref{lem:pgd-rate-smax}.

Lemma~\ref{lem:pgd-rate-smax} analyzes the convergence rate of projected gradient descent, 
when we use it to minimize a smooth non-convex function with constraints.
Lemma~\ref{lem:pgd-rate-smax} follows directly from 
Lemmas~\ref{lem:pqd-smax-reduction}~and~\ref{lem:pgd-smax-g-nablaf}.

Lemma~\ref{lem:pqd-smax-reduction} defines a ``truncated gradient'' mapping $g$ 
and relates the progress in the $\tau$-th iteration with $\normtwo{g(x_\tau)}^2$.
Because we cannot keep decreasing $f(x)$, we know that after many iterations, 
there exists some $\tau$ such that $\normtwo{g(x_\tau)}$ is very small.
Lemma~\ref{lem:pgd-smax-g-nablaf} shows that if $\normtwo{g(x_\tau)}$ is very small, 
that is, if projected gradient descent moves very little between $x_\tau$ and $x_{\tau+1}$, then $x_{\tau+1}$ 
is an approximate stationary point.

\begin{lemma}
\label{lem:pqd-smax-reduction}
Fix a convex set $\KK$.
Suppose $f$ is $\beta$-smooth on $\KK$ and $0 \le f(x) \le B$ for all $x \in \KK$.
Suppose we run projected gradient descent with step size $\eta = \frac{1}{\beta}$ starting from an arbitrary $x_0 \in \KK$, i.e.,
\[
x_{\tau+1} = \Pi_{\KK}\left(x_{\tau} - \eta \nabla f(x_\tau) \right) \;,
\]
where $\Pi_{\KK}$ is the $\ell_2$-projection onto $\KK$.
Then we have that
\[
\min_{0 \le \tau < T} \frac{1}{\eta}\normtwo{\Pi_{\KK}\left(x_\tau - \eta \nabla f(x_\tau)\right) - x_\tau} \le \sqrt{\frac{2 \beta B}{T}} \; .
\]
\end{lemma}
\begin{proof}
Define the mapping
\[
g(x) = \frac{x - \Pi_{\KK}(x - \eta \nabla f(x))}{\eta} \; .
\]

Let $y_{\tau+1} = x_\tau - \eta \nabla f(x_\tau)$.  Notice that $x_{\tau+1} = \Pi_{\KK}(y_{\tau+1}) = x_s - \eta g(x_\tau)$.

By the convexity of $\KK$, we have
\[
(x_{\tau+1} - x_\tau)^\top (x_{\tau+1} - y_{\tau+1}) \le 0 \; ,
\]
which is equivalent to
\[
\nabla f(x_\tau)^\top (x_{\tau+1} - x_\tau) \le g(x_\tau)^\top (x_{\tau+1} - x_\tau) \; .
\]

Using the quadratic upper bound combined with the above inequality, we have
\begin{align*}
f(x_{\tau+1})
&\le f(x_\tau) + \nabla f(x_\tau)^\top (x_{\tau+1} - x_\tau) + \frac{\beta}{2} \normtwo{x_{\tau+1} - x_\tau}^2 \\
&\le f(x_\tau) + g(x_\tau)^\top (x_{\tau+1} - x_\tau) + \frac{\beta}{2} \normtwo{x_{\tau+1} - x_\tau}^2 \\
&= f(x_\tau) - \eta \normtwo{g(x_\tau)}^2 + \frac{\eta^2 \beta}{2} \normtwo{g(x_\tau)}^2 \\
&= f(x_\tau) - \frac{1}{2\beta} \normtwo{g(x_\tau)}^2 \; .
\end{align*}
Therefore, after $T$ iterations, we have
\[
\min_{0 \le \tau < T} \normtwo{g(x_\tau)}^2 \le \frac{1}{T} \sum_{\tau = 0}^{T-1} \normtwo{g(x_\tau)}^2 \le \frac{2 \beta}{T} \left(f(x_0) - f(x_T)\right) \le \frac{2 \beta B}{T} \; . \qedhere
\]
\end{proof}

\begin{lemma}
\label{lem:pgd-smax-g-nablaf}
Consider the same setting as in Lemma~\ref{lem:pqd-smax-reduction}.
Define the tangent cone of $\KK$ at a point $x \in \KK$ as $\CC_\KK(x) = \mathrm{cone}(\KK - \{x\})$.
If for some $\tau$ we have
\[
\normtwo{\Pi_{\KK}\left(x_\tau - \eta \nabla f(x_\tau)\right) - x_\tau} \le \frac{\delta}{2} \; ,
\]
then for all unit vector $u \in \CC_\KK(x)$,
\[
\nabla f(x_{\tau+1})^\top u \le \delta \; .
\]
\end{lemma}
\begin{proof}
By the convexity of $\KK$, we know that for any $z \in \KK$,
\[
(y_{\tau+1} - x_{\tau+1})^\top (z - x_{\tau+1}) \le 0 \; .
\]
Consequently, for any $u \in \CC_\KK(x_{\tau+1})$, we have
\[
(y_{\tau+1} - x_{\tau+1})^\top u \le 0 \; ,
\]
which is equivalent to
\[
-\nabla f(x_\tau)^\top u \le -g(x_\tau)^\top u \; .
\]
Using the fact that $u$ is a unit vector together with the above inequality, we get
\begin{align*}
-\nabla f(x_{\tau+1})^\top u
&\le -\nabla f(x_{\tau+1})^\top u + \nabla f(x_\tau)^\top u - g(x_\tau)^\top u \\
&\le \normtwo{f(x_{\tau+1}) - \nabla f(x_\tau)} + \normtwo{g(x_\tau)} \\
&\le \beta \normtwo{x_{\tau+1} - x_\tau} + \normtwo{g(x_\tau)} \\
&= 2 \normtwo{g(x_\tau)} \le \delta \; . \qedhere
\end{align*}
\end{proof}

\begin{proof}[Proof of Lemma~\ref{lem:pgd-rate-smax}]
As in Algorithm~\ref{alg:smax}, we run projected gradient descent, track the value of $\normtwo{g(x_\tau)}$ in each iteration, and return the $x_\tau$ that has the minimum $\normtwo{g(x_\tau)}$.
Combining Lemmas~\ref{lem:pqd-smax-reduction}~and~\ref{lem:pgd-smax-g-nablaf}, if we want a $\delta$-stationary point, we should set $T$ such that
$\sqrt{2 \beta B / T} \le \delta / 2$, i.e., $T \ge 8 \beta B \delta^{-2} = O(\beta B \delta^{-2})$.
\end{proof}

\subsection{Proofs of Lemmas~\ref{lem:smax-smoothness}~and~\ref{lem:smax-maxvalue}}

In this subsection, we bound from above the smoothness and maximum value of the softmax objective.

For these two lemmas, we can assume without loss of generality that no input samples have very large $\ell_2$-norm.
This is because we can perform a standard preprocessing step that centers the input samples 
at the coordinate-wise median, which does not affect our mean estimation task.
We then throw away all samples that are $\Omega(\sqrt{d \log d})$ far from the coordinate-wise median.
With high probability, the coordinate-wise median and all good samples are $O(\sqrt{d \log d})$ far from the true mean.
Assuming this happens, then no good samples are thrown away and all remaining samples satisfies $\max_i \normtwo{X_i} = O(\sqrt{d \log d})$.
Consequently, we have $\normtwo{\mu_w} = O(\sqrt{d \log d})$ for any $w \in \Delta_{N, \eps}$.

\begin{proof}[Proof of Lemma~\ref{lem:smax-smoothness}]
We proceed to bound from above the spectral norm of the Hessian of $f$.
Recall that $X \in \R^{d \times N}$ and the partial derivative of $f$ with respect to $w_i$ is
\[
\nabla f(w)_i
= X_i^\top Y X_i - 2 X_i^\top Y \mu_w = \left(X_i X_i^\top - X_i \mu_w^\top - \mu_w X_i^\top\right) \bullet Y,
\]
where $Y = \frac{\exp(\rho \Sigma_w)}{\tr \exp (\rho \Sigma_w)}$ is a PSD matrix.
Observe that $Y \succeq 0$, $\tr(Y) = 1$, and $Y$ depends on $w$.

We can compute the $(i,j)$-th entry in the Hessian matrix of $f$, as follows
\[
\nabla^2 f(w)_{i,j} = \frac{d f(w)_i}{d w_j} = \left(X_i X_i^\top - X_i \mu_w^\top - \mu_w X_i^\top\right) \bullet \frac{d Y}{d w_j} - \left(X_i X_j^\top + X_j X_i^\top\right) \bullet Y \;.
\]

By the chain rule, we have
\begin{align*}
\frac{d Y}{d w_j}
&= \frac{1}{\tr(\exp(\rho \Sigma_w))^2}\left[\frac{d \exp(\rho \Sigma_w)}{d w_j} \tr(\exp(\rho\Sigma_w)) - \frac{d \tr(\exp(\rho\Sigma_w))}{d w_j} \exp(\rho\Sigma_w) \right] \\
&= \frac{1}{\tr(\exp(\rho \Sigma_w))}\left[\frac{d \exp(\rho \Sigma_w)}{d w_j}  - \frac{d \tr(\exp(\rho\Sigma_w))}{d w_j} \cdot Y \right] \; .
\end{align*}

Using Lemma~\ref{lem:d-mat-exp} to compute the derivative of matrix exponential, we have
\begin{align*}
\frac{d Y}{d w_j}
&= \frac{1}{\tr(\exp(\rho \Sigma_w))}\left[\frac{d \exp(\rho \Sigma_w)}{d w_j}  - \frac{d \tr(\exp(\rho\Sigma_w))}{d w_j} \; Y \right] \\
&= \frac{1}{\tr \exp(\rho \Sigma_w)} \left[ \int_{\alpha=0}^1 \exp(\alpha \rho \Sigma_w) \frac{d (\rho \Sigma_w)}{d w_j} \exp((1-\alpha)\rho\Sigma_w) d \alpha - \left(\frac{d (\rho \Sigma_w)}{d w_j} \bullet \exp(\rho\Sigma_w) \right) Y \right] \\
&= \frac{\rho}{\tr \exp(\rho \Sigma_w)} \int_{\alpha=0}^1 \exp(\alpha \Sigma_w) \frac{d \Sigma_w}{d w_j} \exp((1-\alpha)\rho\Sigma_w) d \alpha - \rho\left(\frac{d \Sigma_w}{d w_j} \bullet Y\right) Y\; .
\end{align*}

Since $\frac{d \Sigma_w}{d w_j} = X_j X_j^\top - X_j \mu_w^\top - \mu_w X_j^\top$, putting it all together, we have,
\begin{align*}
&\quad \nabla^2 f(w)_{i,j}
= -\left(X_i^\top Y (X_i - 2\mu_w)\right)\left(X_j^\top Y (X_j - 2\mu_w)\right) - 2 X_i^\top Y X_j \\
&\qquad \qquad + \frac{\rho}{\tr \exp(\rho \Sigma_w)} \cdot \\
& \int_{\alpha=0}^1 \tr\left( \left(X_i X_i^\top - X_i \mu_w^\top - \mu_w X_i^\top\right) \exp(\alpha \rho \Sigma_w) \left(X_j X_j^\top - X_j \mu_w^\top - \mu_w X_j^\top\right) \exp((1-\alpha)\rho\Sigma_w) \right) d \alpha \;.
\end{align*}

Let $R = \max(\normtwo{\mu_w}, \max_i \normtwo{X_i})$.
From the preprocessing step, we know that $R = \tilde O(d^{1/2})$.
Using this fact, we obtain
\[
\abs{\nabla^2 f(w)_{i,j}} \le 9 R^4 + 2 R^2 + 9 \rho R^4 = \tilde O(\rho d^2) \;.
\]
This is because the first term can be bounded from above by
\begin{align*}
- \left(X_i^\top Y (X_i - 2\mu_w)\right)\left(X_j^\top Y (X_j - 2\mu_w)\right) 
&\le \normtwo{X_i} \normtwo{Y} \normtwo{X_i - 2 \mu_w} \normtwo{X_j} \normtwo{Y} \normtwo{X_j - 2 \mu_w} \\
&\le 9 R^4 \;.
\end{align*}
Similarly, the second term is at most $2 R^2$.
The third term can be split into $9$ terms of the form
\begin{align*}
&\quad \frac{\rho}{\tr \exp(\rho \Sigma_w)} \int_{\alpha=0}^1 \tr\left( \left(X_i X_i^\top \right) \exp(\alpha \rho \Sigma_w) \left(X_j X_j^\top \right) \exp((1-\alpha)\rho\Sigma_w) \right) d \alpha \\
&= \frac{\rho}{\tr \exp(\rho \Sigma_w)} \int_{\alpha=0}^1 \left(X_i^\top \exp(\alpha \rho \Sigma_w) X_j\right) \left(X_j^\top \exp((1-\alpha)\rho\Sigma_w) X_i \right) d \alpha \\
&\le \frac{\rho}{\tr \exp(\rho \Sigma_w)} \int_{\alpha=0}^1 \normtwo{X_i} \normtwo{\exp(\alpha \rho \Sigma_w)} \normtwo{X_j} \normtwo{X_j} \normtwo{\exp((1-\alpha)\rho\Sigma_w)} \normtwo{X_i} d \alpha \\
&= \frac{\rho}{\tr \exp(\rho \Sigma_w)} \cdot R^4 \cdot \normtwo{\exp(\rho \Sigma_w)} \le \rho R^4 \; .
\end{align*}

To conclude the proof, we bound from above the smoothness parameter by the spectral norm of the Hessian matrix.
For any $w \in \Delta_{N,2\eps}$,
\[
\normtwo{\nabla^2 f(w)} \le N \cdot \max_{ij} \abs{\nabla^2 f(w)_{ij}} \le O(N \rho d^2) = \tilde O(N d^2 / \eps) \;,
\]
where the last step uses that $\rho = \ln d / \eps$.
\end{proof}

\begin{proof}[Proof of Lemma~\ref{lem:smax-maxvalue}]
Fix any $w \in \Delta_{N,2\eps}$.
By Corollary~\ref{cor:smax-max} and our choice of $\rho = \frac{\ln d}{\eps}$, we have
\[
f(w) = \smax_{\rho}(\Sigma_w) \le \lambda_{\max}(\Sigma_w) + \eps.
\]
Therefore, it is sufficient to bound from above $\lambda_{\max}(\Sigma_w)$ by $O(d \log d)$.

The preprocessing step guarantees that all samples have $\ell_2$-norm at most $\tilde O(d^{1/2})$, consequently, 
the weighted empirical mean $\mu_w$ has $\ell_2$-norm is at most $\tilde O(d^{1/2})$ as well.
Consequently,
\begin{align*}
\normtwo{\Sigma_w}
  &= \normtwo{\sum_{i=1}^N w_i (X_i - \mu_w)(X_i - \mu_w)^\top} \\
  &\le \sum_{i=1}^N w_i \normtwo{(X_i - \mu_w)(X_i - \mu_w)^\top}
  \le \max_{i \in [N]} \normtwo{X_i - \mu_w}^2 \le \tilde O(d) \;. 
\end{align*}
The proof is now complete.
\end{proof} 

\end{document}